\documentclass{article}

\usepackage{microtype}
\usepackage{graphicx}
\usepackage{subfigure}
\usepackage{booktabs} 

\usepackage{hyperref}

\usepackage[accepted]{icml2024}

\usepackage{amsmath}
\usepackage{amssymb}
\usepackage{mathtools}
\usepackage{amsthm}
\usepackage{adjustbox}
\usepackage{enumitem}
\usepackage{algorithmic}
\usepackage{graphicx}
\usepackage{wrapfig}

\usepackage[capitalize,noabbrev]{cleveref}

\theoremstyle{plain}
\newtheorem{theorem}{Theorem}[section]
\newtheorem{proposition}[theorem]{Proposition}

\theoremstyle{definition}
\newtheorem{definition}[theorem]{Definition}

\theoremstyle{remark}
\newtheorem{remark}[theorem]{Remark}

\usepackage[textsize=tiny]{todonotes}
\usepackage{xspace}
\usepackage{amsthm}
\usepackage{thmtools,thm-restate}

\icmltitlerunning{Pairwise Alignment Improves Graph Domain Adaptation}

\newcommand{\dG}{\mathcal{G}}
\newcommand{\dS}{\mathcal{S}}

\newcommand{\dV}{\mathcal{V}}
\newcommand{\dE}{\mathcal{E}}

\newcommand{\dX}{\mathcal{X}}
\newcommand{\dT}{\mathcal{T}}
\newcommand{\dU}{\mathcal{U}}
\newcommand{\dN}{\mathcal{N}}
\newcommand{\dH}{\mathcal{H}}
\newcommand{\dY}{\mathcal{Y}}

\newcommand{\mX}{\mathbf{X}}
\newcommand{\mY}{\mathbf{Y}}
\newcommand{\mA}{\mathbf{A}}

\newcommand{\mB}{\mathbf{B}}
\newcommand{\mC}{\mathbf{C}}
\newcommand{\mH}{\mathbf{H}}

\newcommand{\mM}{\mathbf{M}}
\newcommand{\mK}{\mathbf{K}}

\newcommand{\mw}{\mathbf{w}}
\newcommand{\mx}{\mathbf{x}}
\newcommand{\my}{\mathbf{y}}

\newcommand{\mh}{\mathbf{h}}

\newcommand{\mD}{\mathbf{D}}
\newcommand{\mE}{\mathbf{E}}

\newcommand{\mSigma}{\boldsymbol{\Sigma}}

\newcommand{\mbeta}{\boldsymbol{\beta}}
\newcommand{\mgamma}{\boldsymbol{\gamma}}

\newcommand{\mmu}{\boldsymbol{\mu}}
\newcommand{\mnu}{\boldsymbol{\nu}}
\newcommand{\malpha}{\boldsymbol{\alpha}}

\newcommand{\mL}{\mathbf{L}}

\newcommand{\mone}{\mathbf{1}}

\newcommand{\bP}{\mathbb{P}}

\newcommand{\bR}{\mathbb{R}}

\newcommand*{\ldbb}{\{\mskip-5mu\{}
\newcommand*{\rdbb}{\}\mskip-5mu\}}
\newcommand{\hy}{\hat{y}}
\newcommand{\hY}{\hat{Y}}

\newcommand{\proj}{Pair-Align\xspace}
\newcommand{\projew}{PA-CSS\xspace}
\newcommand{\projlw}{PA-LS\xspace}
\newcommand{\projb}{PA-BOTH\xspace}
\renewcommand{\epsilon}{\varepsilon}

\begin{document}

\twocolumn[
\icmltitle{Pairwise Alignment Improves Graph Domain Adaptation}



\icmlsetsymbol{equal}{*}

\begin{icmlauthorlist}
\icmlauthor{Shikun Liu}{gt}
\icmlauthor{Deyu Zou}{a}
\icmlauthor{Han Zhao}{uiuc}
\icmlauthor{Pan Li}{gt}
\end{icmlauthorlist}

\icmlaffiliation{gt}{Department of Electrical and Computer Engineering, Georgia Institute of Technology, Georgia, USA}
\icmlaffiliation{a}{School of Data Science, University of Science and Technology of China, Hefei, China}
\icmlaffiliation{uiuc}{Department of Computer Science, University of Illinois Urbana-Champaign, Champaign, USA}

\icmlcorrespondingauthor{Shikun Liu}{shikun.liu@gatech.edu}
\icmlcorrespondingauthor{Pan Li}{panli@gatech.edu}

\icmlkeywords{Machine Learning, Graph Domain Adaptation}

\vskip 0.3in
]



\printAffiliationsAndNotice{} 

\begin{abstract}

Graph-based methods, pivotal for label inference over interconnected objects in many real-world applications, often encounter generalization challenges, if the graph used for model training differs significantly from the graph used for testing. This work delves into Graph Domain Adaptation (GDA) to address the unique complexities of distribution shifts over graph data, where interconnected data points experience shifts in features, labels, and in particular, connecting patterns. We propose a novel, theoretically principled method, Pairwise Alignment (\proj) to counter graph structure shift by mitigating conditional structure shift (CSS) and label shift (LS). \proj uses edge weights to recalibrate the influence among neighboring nodes to handle CSS and adjusts the classification loss with label weights to handle LS. Our method demonstrates superior performance in real-world applications, including node classification with region shift in social networks, and the pileup mitigation task in particle colliding experiments. For the first application, we also curate the largest dataset by far for GDA studies. Our method shows strong performance in synthetic and other existing benchmark datasets. \footnote{
Our code and data are available at: \url{https://github.com/Graph-COM/Pair-Align}}
\end{abstract}

\section{Introduction}
\label{sec:intro}

Graph-based methods are commonly used to enhance label inference for interconnected objects by utilizing their connection patterns in many real-world applications~\cite{jackson2008social,szklarczyk2019string,shlomi2020graph}. Nonetheless, these methods often encounter generalization challenges, as the objects that lack labels and require inference may originate from domains that differ significantly from those with abundant labeled data, thereby exhibiting distinct interconnecting patterns. For instance, in fraud detection within financial networks, label acquisition may be constrained to specific network regions due to varying international legal frameworks and diverse data collection periods
~\cite{wang2019semi, dou2020enhancing}. Another example is particle filtering for Large Hadron Collider (LHC) experiments~\cite{highfield2008large}, where reliance on simulation-derived labeled data poses a challenge. These simulations may not accurately capture the nuances of real-world experimental conditions, potentially leading to discrepancies in label inference performance when applied to actual experiment scenarios~\cite{li2022semi, komiske2017pileup}.


Graph Neural Networks (GNNs) have recently demonstrated remarkable effectiveness in utilizing object interconnections for label inference tasks~\cite{kipf2016semi, hamilton2017inductive, velivckovic2018graph}. However, their effectiveness is often hampered by the vulnerability to variations in data distribution~\cite{ji2023drugood, ding2021closer, koh2021wilds}. This has sparked significant interest in developing GNNs capable of generalization from one domain (source domain $\dS$) to another, potentially different domain (target domain $\dT$). This field of study, known as graph domain adaptation (GDA), is gaining increasing attention. GDA distinguishes itself from the traditional domain adaptation setting, primarily because the data points in GDA are interlinked rather than independent. This non-IID nature of graph data renders traditional domain adaptation techniques suboptimal when applied to graphs. The distribution shifts in features, labels, and connecting patterns between objects may significantly impact the adaptation/generalization accuracy. Despite the recent progress made in GDA~\cite{wu2020unsupervised,you2023graph,zhu2021shift,liu2023structural}, current solutions still struggle to tackle the various shifts prevalent in real-world graph data. We provide a detailed discussion of the limitations of existing GDA methods in Section~\ref{sec:related}.

\begin{figure}[t]
\begin{center}
\vspace{-1mm}
\centerline{\includegraphics[trim={0.8cm 0.2cm 0cm 0cm},clip,width=1\columnwidth]{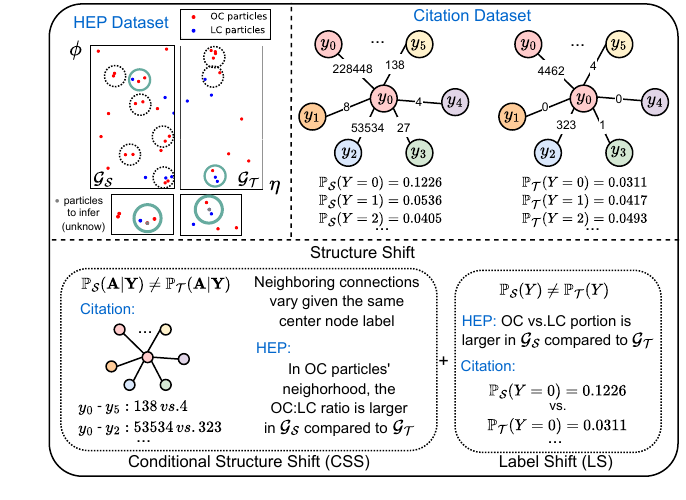}}
\vspace{-2mm}
\caption{\small{We illustrate structure shifts in real-world datasets: a) The HEP dataset in pileup mitigation tasks~\cite{bertolini2014pileup} has a shift in PU levels (change in the number of other collisions (OC) around the leading collision (LC) for proton-proton collision events), where $\dG_\dS$ is in PU30 and $\dG_\dT$ is in PU10; 
Here, in the green circles, the center nodes in grey are the particles whose labels are to be inferred. They have different ground-truth labels but the same neighborhood that includes one OC and one LC particle.
b) The citation MAG dataset shifts in regions, where the source graph contains papers in the US and the target graph contains papers in German. More statistics on graph distribution shift from real-world examples can be found in Appendix~\ref{app:shift_stats}}.}
\label{fig:shift}
\end{center}
\vspace{-9mm}
\end{figure}

This work conducts a systematic study of the distinct challenges present in GDA and proposes a novel method, named Pairwise Alignment (\proj) to tackle graph structure shift for node prediction tasks. Combined with feature alignment methods offered by traditional non-graph DA techniques~\cite{ganin2016domain,tachet2020domain}, \proj can in principle address a wide range of distribution shifts in graph data. 

Our analysis begins with examining a graph with its adjacency matrix $\mA$ and node labels $\mY$. We observe that graph structure shift ($\bP_\dS(\mA,\mY) \neq \bP_\dT(\mA,\mY)$) typically manifests as either conditional structure shift (CSS) or label shift (LS), or a combination of both. CSS refers to the change in neighboring connections among nodes within the same class ($\bP_\dS(\mA|\mY) \neq \bP_\dT(\mA|\mY)$) whereas LS denotes changes in the class distribution of nodes ($\bP_\dS(\mY) \neq \bP_\dT(\mY)$). These shifts are illustrated in Fig.~\ref{fig:shift} via examples in HEP and social networks, and are justified by statistics from several real-world applications. 

In light of the two types of shifts, the \proj method aims to estimate and subsequently mitigate the distribution shift in the neighboring nodes' representations for any given node class $c$. To achieve this, \proj employs a bootstrapping technique to recalibrate the influence of neighboring nodes in the message aggregation phase of GNNs. This strategic reweighting is key to effectively countering CSS.\ Concurrently, \proj calculates label weights to alleviate disparities in the label distribution between source and target domains (addressing LS) by adjusting the classification loss. \proj is depicted in Figure~\ref{fig:pipeline}.

\begin{figure*}[t]
\begin{center}
\centerline{\includegraphics[trim={1.4cm 0.3cm 0.1cm 1.2cm},clip,width=0.8\linewidth]{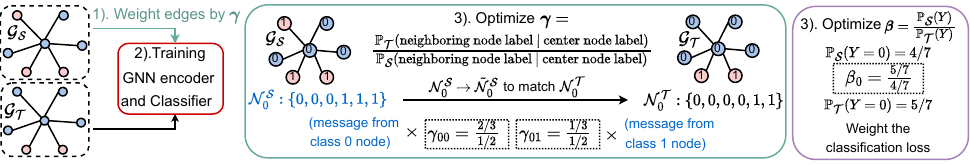}}
\vspace{-2mm}
\caption{\small{The pipeline contains modules in handling CSS with edge weights $\mgamma$ and handling LS with label weights $\mbeta$} 
}
\label{fig:pipeline}
\end{center}
\vspace{-2.3em}
\end{figure*}

To demonstrate the effectiveness of our pipeline, we curate the regional MAG data that partitions large citation networks according to the regions where papers got published~\cite{hu2020open, wang2020microsoft} to simulate the region shift. To the best of our knowledge, this is the largest dataset (of $\approx$ 380k nodes, 1.35M edges) to study GDA with data retrieved from the real-world database. 
We also include other graph data with shifts, like the pileup mitigation task studied in~\citet{liu2023structural}. Our method shows strong performance in these two applications. Moreover, our method also outperforms baselines significantly in synthetic datasets and other real-world benchmark datasets.

\section{Preliminaries and Related Works}  
\label{sec:prelim}

\subsection{Notation and The Problem Setup}
\label{subsec:setup}

We use capital letters, e.g., $Y$ to denote scalar random variables, and lower-case letters, e.g., $y$ to denote their realizations. The bold counterparts are used for their vector-valued correspondences, e.g., $\mY, \my$, and the calligraphic letters, e.g. $\dY$, are for the value spaces. We always use capital letters to denote matrices.
Let $\bP$ denote a distribution, whose subscript $\dU \in \{\dS, \dT\}$ indicates the domain it depicts, e.g. $\bP_\dS(Y)$. The probability of a realization, e.g. $Y=y$, can then be denoted as $\bP_\dS(Y=y)$.

\textbf{Graph Neural Networks (GNNs).} 
We use $\dG=(\dV,\dE,\mx)$ to denote a graph with the node set $\dV$, the edge set $\dE$ and node features $\mx=[\cdots x_u\cdots]_{u\in \dV}$. We focus on undirected graphs where the graph structure can also be represented as a symmetric adjacency matrix $\mA$ where the entries $A_{uv}= A_{vu} = 1$ when nodes $u,v$ form an edge and otherwise 0. 
GNNs take $\mA$ and $\mx$ as input and output node representations $\{h_u, \forall u\in \dV\}$. The standard GNNs~\cite{hamilton2017inductive} has a message-passing procedure. Specifically, with $h_u^{(1)}=x_u$,  for each node $v$ and each layer $k\in [L]:= \{1,\ldots,L\}$, 
\begin{equation}
    h_u^{(k+1)}=\text{UPT}\,(h_u^{(k)}, \text{AGG}\,(\ldbb h_v^{(k)}:v \in \dN_u\rdbb)), 
    \label{eq:GNN}
\end{equation}
where $\dN_v$ denotes the set of neighbors of node $v$ and $\ldbb\cdot\rdbb$ denotes a multiset. The AGG function aggregates messages from the neighbors, and the UPT function updates the node representations. The last-layer node representation $h_u^{(L)}$ is used to predict the label $y_u\in\dY$ in node classification tasks.

\textbf{Domain Adaptation (DA).} In DA, each domain $\dU\in \{S, T\}$ has its own joint feature and label distribution $\bP_\dU(X,Y)$. In the unsupervised setting, we have access to labeled source data $\{(x_i, y_i)\}_{i=1}^N$ and unlabeled target data $\{(x_i)\}_{i=1}^M$ IID sampled from the source and target domain respectively. The model comprises a feature encoder $\phi: \dX\rightarrow \dH$ and a classifier $g: \dH \rightarrow \dY$, with classification error in domain $\dU$ denoted as $\epsilon_\dU(g\circ\phi) = \bP_\dU(g(\phi(X)) \neq Y)$. The objective is to train the model with available data to minimize target error $\epsilon_\dT(g\circ\phi)$ when predicting target labels. A popular DA strategy is to learn domain-invariant representation, ensuring similar $\bP_\dS(H)$ and $\bP_\dT(H)$ and minimizing the source error $\epsilon_\dS(g\circ\phi)$ to retain classification capability simultaneously~\citep{zhao2019learning}. This is achieved through regularization of distance measures~\cite{long2015learning, zellinger2016central} or adversarial training~\cite{ganin2016domain, tzeng2017adversarial,zhao2018adversarial}. 

\textbf{Graph Domain Adaptation (GDA).} When extending unsupervised DA to the graph-structured data, we are given a source graph $\dG_{\dS} = (\dV_{\dS}, \dE_{\dS}, \mx_{\dS})$ with node labels $\my_{\dS}$ and a target graph $\dG_{\dT} = (\dV_{\dT}, \dE_{\dT}, \mx_{\dT})$. The specific distribution and shifts in graph-structured data will be defined in Sec.\ref{sec:method}. The objective is similar to DA as to minimize the target error, but with the encoder $\phi$ switched to a GNN to predict node labels $\my_{\dT}$ in the target graph. 

\subsection{Related Works and Existing Gaps} 
\label{sec:related}
GDA research falls into two main categories, aiming at addressing domain adaptation for node and graph classification tasks respectively. Often, graph-level GDA problems can view each graph as an independent sample, allowing extension of previous non-graph DA techniques to graphs, such as causal inference~\cite{rojas2018invariant, peters2017elements} (more are reviewed in Appendix~\ref{app:morerelated}). Conversely, node-level GDA presents challenges due to the interconnected nodes. Previous works mainly leveraged node representations as intermediaries to address these challenges.

The dominant idea of existing work on node-level GDA focused on aligning the marginal distributions of node representations, mostly over the last layer $\mh^{(L)}$, across two graphs inspired by the domain invariant learning in DA~\citep{liao2021information}. Some of them adopted adversarial training, such as~\cite{dai2022graph, zhang2019dane, shen2020adversarial}. UDAGCN~\cite{wu2020unsupervised} calculated the point-wise mutual information and inter-graph attention to exploit local and global consistency on top of the adversarial training. Other works were motivated by regularizing different distance measures. \citet{zhu2021shift} regularized over the central moment discrepancy~\cite{zellinger2016central}. \citet{you2023graph} minimized the Wasserstein-1 distance between the distributions of node representations and controlled GNN Lipschitz via regularizing graph spectral properties. \citet{wu2023non} introduced graph subtree discrepancy inspired by the WL subtree kernel~\cite{shervashidze2011weisfeiler} and suggested regularizing node representations after each layer of GNNs. Furthermore, \citet{zhu2022shift,zhu2023explaining} recognized that there could also be a shift in the label distribution, so they proposed to align the distribution of label/pseudo-label in addition to the marginal node representation.

Nonetheless, the marginal alignment methods above are inadequate when dealing with the structure shift consisting of CSS and LS. Firstly, these methods are flawed under LS. Based on $\bP_\dU(H^{(L)}) = \sum_{Y}\bP_\dU(H^{(L)}|Y)\bP_\dU(Y)$, 
even if the marginal alignment $\bP_\dS(H^{(L)}) = \bP_\dT(H^{(L)})$ is achieved, the conditional node representations  will still mismatch $\bP_\dS(H^{(L)}|Y)\neq  \bP_\dT(H^{(L)}|Y)$ under the LS, which induces more prediction error~\cite{zhao2019learning, tachet2020domain}. Secondly, they are suboptimal under CSS. In particular, consider the HEP example in Fig.~\ref{fig:shift} (the particles in the two green circles) where CSS may yield the case that the label of the center particle (node) shifts, albeit with an unchanged neighborhood distribution. In this case, methods using a shared GNN encoder for marginal alignment definitely fail to make the correct prediction.  

\citet{liu2023structural} have recently analyzed this issue by using an example based on contextual stochastic block model (CSBM)~\cite{deshpande2018contextual} (defined in Appendix~\ref{app:CSBM}).  

\begin{proposition}\cite{liu2023structural}
\label{eg}
Suppose the source and target graphs are generated from the CSBM model of $n$ nodes with the same label distributions and node feature distributions. The edge connection probabilities are set to present a conditional structure shift $\bP_\dS(\mA|\mY) \neq \bP_\dT(\mA|\mY)$ and showcase the example that the ground truth label of the center node changes given the same neighborhood distribution. Then, 
suppose a GNN encoder $\phi$ is shared across two domains, the target classification error $\epsilon_{\dT}(g\circ\phi)$ can be lower bounded by 0.25, where $g$ is the classifier. However, the GNN encoder $\phi$, if allowed to be adjusted according to the domains, can achieve   $\epsilon_{\dT}(g\circ\phi)\rightarrow 0$ as $n\rightarrow \infty$.
\vspace{-2mm}
\end{proposition}

To tackle this issue, \citet{liu2023structural} proposed the StruRW method to reweight edges in the source graph based on weights derived from the CSBM model. However, StruRW still suffers from many issues. 
We will provide a more detailed comparison with StruRW in Sec.~\ref{sec:compstrurw}. To the best of our knowledge, our method is the first effort to address both CSS and LS in a principled way.

\section{Pairwise Alignment for Structure Shift}
\label{sec:method}
We first define shifts in graphs as feature shift and structure shift, the latter includes both the Conditional Structure Shift (CSS) and the Label Shift (LS). Then, we analyze the objective of solving structure shift and propose our pairwise alignment algorithm that handles both CSS and LS.

\subsection{Distribution Shifts in Graph-structured Data}
Sec.~\ref{sec:related} shows the sub-optimality of enforcing marginal node representation alignment under structure shifts. In fact, the necessity of conditional distribution alignment $\bP_\dS(H|Y) = \bP_\dT(H|Y)$ to deal with feature shift $\bP_\dS(X|Y) \neq \bP_\dT(X|Y)$ has been explored in non-graph scenarios, where $X$ denotes a feature vector and $H$ is the representation after $X$ passes through the encoder, i.e., $H=\phi(X)$. 
Early efforts such as \citet{zhang2013domain, gong2016domain} assumed that the shift in conditional representations from domain $\dS$ to domain $\dT$ follows a linear transformation and optimized conditional alignment by introducing an extra linear transformation to the source domain encoder to enhance conditional alignment $\bP_\dS(H|Y) = \bP_\dT(H|Y)$. 
Subsequent works learned the representations with adversarial training to enforce conditional alignment by aligning the joint distribution over the label predictions and representations~\cite{long2018conditional, cicek2019unsupervised}. Later, some works additionally considered label shift~\cite{tachet2020domain, liu2021adversarial} and proposed to match the label weighted $\bP_\dS^{\text{lw}}(H)$ with $\bP_\dT(H)$ with label weights estimated following~\citet{lipton2018detecting}. 

In light of the limitations of existing works and the effort in non-graph DA research, it becomes clear that marginal alignment of node representations is insufficient for GDA, which underscores the importance of achieving conditional node representation alignment. 

To address various distribution shifts for GDA in principle, we first decouple the potential distribution shifts in graph data by defining feature shift and structure shift in terms of conditional distributions and label distributions. Our data generation process can be characterized by the following model: $\mX\leftarrow \mY\rightarrow \mA$, where labels are drawn at each node first, and then edges as well as features at each node are generated. Under this model, we define the following feature shift, which denotes the change of the conditional feature generation process given the labels. 
\begin{definition}[Feature Shift]
Assume the node features $x_u$, $u\in\dV$ are IID sampled from $\bP(X|Y)$ given node labels $y_u$. Therefore, the conditional distribution of $\mx|\my$, $\bP(\mX = \mx|\mY = \my) =\prod_{u\in \dV} \bP(X = x_u|Y = y_u)$. 
The \underline{feature shift} is then defined as $\bP_\dS(X|Y)\neq \bP_\dT(X|Y)$. 
\end{definition}

\begin{definition}[Structure Shift] Given the joint distribution of the adjacency matrix and node labels $\bP(\mA,\mY)$. The \underline{Structure Shift} is defined as $\bP_\dS(\mA,\mY) \neq \bP_\dT(\mA, \mY)$. 
With decomposition as $\bP_\dU(\mA, \mY) = \bP_\dU(\mA|\mY) \bP_\dU(\mY)$, it results in \underline{Conditional Structure Shift} (CSS) and \underline{Label Shift} (LS):
\begin{itemize}[noitemsep,nolistsep]
    \item CSS: $\bP_\dS(\mA|\mY) \neq \bP_\dT(\mA|\mY)$
    \item LS: $\bP_\dS(\mY) \neq \bP_\dT(\mY)$
\end{itemize}
\label{def:GSS}
\vspace{-1.5mm}
\end{definition}
As shown in Fig.~\ref{fig:shift}, structure shift consisting of CSS and LS widely exists in real-world applications. Feature shift here, which is equivalent to the conditional feature shift in non-graph literature, can be addressed by adapting conventional conditional shift methods. So, later, we assume that feature shift has been addressed, i.e., $\bP_\dS(X|Y) = \bP_\dT(X|Y)$. 

In contrast, structure shift is unique to graph data due to the non-IID nature caused by node interconnections. Moreover, the learning of node representations is intrinsically linked to the graph structure as the GNN encoder takes $\mA$ as input. Therefore, even if after one layer of GNN, $\bP_\dS(H^{(k)}|Y) = \bP_\dT(H^{(k)}|Y)$ is achieved, CSS could still lead to misalignment of conditional node representation distributions in the next layer $\bP_\dS(H^{(k+1)}|Y) \neq \bP_\dT(H^{(k+1)}|Y)$. Accordingly, a tailored algorithm is needed to remove this effect of CSS, which, when combined with techniques for LS, can effectively resolve the structure shift.

\subsection{Addressing Conditional Structure Shift}
\label{sec:CSS}
To remove the effect of CSS under GNN, the objective is to guarantee $\bP_\dS(H^{(k+1)}|Y) = \bP_\dT(H^{(k+1)}|Y)$ given $\bP_\dS(H^{(k)}|Y) = \bP_\dT(H^{(k)}|Y)$. Considering one layer of GNN encoding in Eq.~\eqref{eq:GNN}: given $\bP_\dS(H^{(k)}|Y) = \bP_\dT(H^{(k)}|Y)$ 
, the mismatch in $k+1$ layer may arise from the distribution shift of the neighboring multiset $\ldbb h_v^{(k)}:v \in \dN_u\rdbb$ given the center node label $y_u$. 
Therefore, the key is to transform the neighboring multisets in the source graph to achieve conditional alignment with the target domain regarding the distributions of such neighboring multisets. 
Our approach first starts with 
a sufficient condition for such conditional alignment.

\begin{restatable}[Sufficient conditions for addressing CSS]{theorem}{mythm}
\label{thm:decompose}
    Given the following assumptions 
    \setlist{nolistsep}
    \begin{itemize}[noitemsep,leftmargin=3mm]
        \item (Conditional Alignment in the previous layer $k$) $\bP_\dS(H^{(k)}|Y) = \bP_\dT(H^{(k)}|Y)$ and $\forall u \in \dV_\dU$, given $Y=y_u$, $h_u^{(k)}$ is independently sampled from $\bP_\dU(H^{(k)}|Y)$.
        \item (Edge Conditional Independence) Given node labels $\my$, edges mutually independently exist in the graph. 
    \end{itemize}
    if there exists a transformation that modifies the neighborhood of node $u$: $\dN_u \rightarrow \Tilde{\dN}_u, \forall u \in \dV_\dS$, such that $\bP_\dS(|\Tilde{\dN}_u||Y_u = i) =\bP_\dT(|\dN_u||Y_u = i)$ and $\bP_\dS(Y_v |Y_u = i, v\in \Tilde{\dN}_u) = \bP_\dT(Y_v|Y_u = i, v\in \dN_u)$, $\forall i\in \dY$, then $\bP_\dS(H^{(k+1)}|Y) = \bP_\dT(H^{(k+1)}|Y)$ is satisfied.
\end{restatable}
\vspace{-0.25em}
\begin{remark} The assumption edge conditional independence essentially assumes an SBM model for the graph structure, which is widely adopted for graph learning algorithm  analysis~\cite{liu2023structural,wei2022understanding}.
\vspace*{-0.5em}
\end{remark}
This theorem reveals that it suffices to align two distributions with the multiset transformation on the source graph: 
1) the distribution of the degree/cardinality of the neighbors $\bP_\dU(|\dN_u| | Y_u)$ and 2) the node label distribution in the neighborhood $\bP_\dU(Y_v|Y_u, v\in \dN_u)$, both conditioned on the center node label $Y_u$. 

\textbf{Multiset Alignment.} Bootstrapping the elements in the multisets can be used to align the two distributions. In the context of GNNs, which typically employ sum/mean pooling functions to aggregate the multisets, such a bootstrapping process can be translated into assigning weights to different neighboring nodes given their labels and the center node's label. 
Moreover, practically, mean pooling is often the preferred choice due to its superior empirical performance, which is also observed in our experiments. 
Aligning the distributions of the node degrees $\bP_\dU(|\dN_u| | Y_u)$ yields negligible impact with mean pooling~\cite{xu2018powerful}. Therefore, our method focuses on aligning the  distribution 
$\bP_\dU(Y_v|Y_u, v\in \dN_u)$, 
in which the edge weights are the ratios of such probabilities across two domains:
\begin{definition}
    Assume $\bP_\dS(Y_v = j|Y_u = i, v\in \dN_u) > 0, \forall i,j \in \dY$, we define $\mgamma \in \bR^{|\dY|\times |\dY|}$ as:
    \vspace{-0.25mm}
    \begin{align*}
        [\mgamma]_{i,j} = \frac{\bP_\dT(Y_v = j|Y_u = i, v\in \dN_u)}{\bP_\dS(Y_v = j|Y_u = i, v\in \dN_u)}, \forall i, j \in \dY
    \end{align*}    
\end{definition}
\vspace{-2.5mm}
where $[\mgamma]_{i,j}$ is the density ratio between the target and source graphs from class-$i$ nodes to class-$j$ nodes. Note that $[\mgamma]_{i,j}\neq [\mgamma]_{j,i}$. 
To differentiate the encoding with and without the adjusted edge weights for the source and target graphs, we denote the operation that first adjusts the edge weights $\mgamma$ and then apply GNN encoding as $\phi_{\mgamma}$ while the one that directly applies GNN encoding as $\phi$. 
By assuming the conditions made in Thm~\ref{thm:decompose} and applying them in an iterative manner for each layer of GNN, the last-layer alignment $\bP_\dS(H^{(L)}|Y) = \bP_\dT(H^{(L)}|Y)$ can be achieved with $\mh^{(L)}_\dS = \phi_{\gamma}(\mx_\dS, \mA_\dS)$ and $\mh^{(L)}_\dT = \phi(\mx_\dT, \mA_\dT)$. Note that based on conditional alignment in the distribution of randomly sampled node representations $\bP_\dS(H^{(L)}|Y) = \bP_\dT(H^{(L)}|Y)$ and under the conditions in Thm~\ref{thm:decompose}, $ \bP_\dS(\mH^{(L)}|\mY) = \bP_\dT(\mH^{(L)}|\mY)$ can also be achieved in the matrix form.

\textbf{$\mgamma$ Estimation.} Till now we explain why edge reweighting using $\mgamma$ can address CSS for GNN encoding, we will detail our pairwise alignment method to obtain $\mgamma$ next. By definition, $\mgamma$ can be decomposed into another two weights.
 \begin{definition}
    Assume $\bP_\dS(Y_u = i, Y_v = j|e_{uv} \in \dE_\dS) > 0, \forall i,j \in \dY$, we define $\mw \in \bR^{|\dY|\times |\dY|}$ and $\malpha \in \bR^{|\dY|\times 1}$ as:
    \vspace{-0.2mm}
    \begin{align*}
        [\mw]_{i,j} &= \frac{\bP_\dT(Y_u = i, Y_v = j|e_{uv} \in \dE_\dT)}{\bP_\dS(Y_u = i, Y_v = j|e_{uv} \in \dE_\dS)}, \\
        [\malpha]_{i} &= \frac{\bP_\dT(Y_u = i|e_{uv} \in \dE_\dT)}{\bP_\dS(Y_u = i|e_{uv} \in \dE_\dS)}, \forall i,j \in \dY
    \end{align*}  
\end{definition}
\vspace{-2.5mm}
and $\mgamma$ can be estimated via
\vspace{-0.5mm}
\begin{align}
    \mgamma = \text{diag}(\malpha)^{-1}\mw
\label{eq:calcgamma}
\end{align}
For domain $\dU$, $\bP_\dU(Y_u, Y_v|e_{uv} \in \dE_\dU)$ is the joint distribution of the label pairs of two nodes that form an edge, which can be computed for domain $\dS$ but not for domain $\dT$. $\bP_\dU(Y_u|e_{uv} \in \dE_\dU)$ can be obtained by marginalizing $\bP_\dU(Y_u, Y_v|e_{uv} \in \dE_\dU)$ over $Y_v$, as $\bP_\dU(Y_u = i|e_{uv} \in \dE_\dU) = \sum_{j\in \dY}\bP_\dU(Y_u=i, Y_v=j|e_{uv} \in \dE_\dU)$.  
Also, it is crucial to differentiate $\bP_\dU(Y_u|e_{uv} \in \dE_\dU)$ from $\bP_\dU(Y)$: the former is the label distribution of the end node conditioned on an edge, 
while the latter is the label distribution of nodes without conditions. Given $\mw$ and two distributions computed over the source graph, $\malpha$ can be derived via
\begin{align}
    [\malpha]_i = \frac{\sum_{j\in \dY}([\mw]_{i,j}\bP_\dS(Y_u = i, Y_v = j|e_{uv} \in \dE_\dS))}{\bP_\dS(Y_u = i|e_{uv} \in \dE_\dS)},
\label{eq:calcalpha}
\end{align}
so next, we proceed to estimate $\mw$ to complete $\mgamma$ calculation. 


\textbf{Pair-wise Alignment.} Note that if $(Y_u, Y_v)$ is viewed as a type for edge $e_{uv}$,  $\bP_\dU(Y_u, Y_v|e_{uv} \in \dE_\dU)$ essentially represents an edge-type distribution. 
In practice, we use \emph{pair-wise} pseudo-label distribution alignment to estimate $\mw$. 

\begin{definition}
    Let $\mSigma \in \bR^{|\dY|^2 \times |\dY|^2}$  denote the matrix that stands for the joint distribution of the predicted types of edges and the true types of edges, and $\mnu\in \bR^{|\dY|^2 \times 1}$ denote the distribution of the predicted types of edges for the target domain, $\forall i, j, i', j' \in \dY$
    \begin{align*}
        [\mSigma]_{ij,i'j'} &= \bP_\dS(\hY_u = i, \hY_v = j, Y_u = i', Y_v = j'| e_{uv} \in \dE_\dS)\\
        [\mnu]_{ij} &= \bP_\dT(\hY_u = i, \hY_v = j| e_{uv} \in \dE_\dT)
    \end{align*}
\end{definition}
\vspace{-1.5mm}

Specifically, similar to~\citet[Lemma 3.2]{tachet2020domain}, Lemma~\ref{lem:optw} shows that $\mw$ can be obtained by 
solving the linear system $\mnu = \mSigma\mw$ if $\bP_\dS(H^{(L)}|Y) = \bP_\dT(H^{(L)}|Y)$ is satisfied. 
\begin{restatable}{lemma}{wlemma}
    If $\bP_\dS(H^{(L)}|Y) = \bP_\dT(H^{(L)}|Y)$ is satisfied, and node representations are conditionally independent of graph structures given node labels, then $\mnu = \mSigma\mw$.
    \label{lem:optw}
\end{restatable}

Empirically, we estimate $\hat{\mSigma}$ and $\hat{\mnu}$ based on  the classifier $g$, where $g(h_u^{(L)})$ denotes the soft label of node $u$. Specifically, 
\begin{gather*}
    [\hat{\mSigma}]_{ij,i'j'} = \frac{1}{|\dE_\dS|}\sum_{e_{uv}\in \dE_\dS, y_u = i', y_v = j'}[g(h_u^{(L)})]_i \times [g(h_v^{(L)})]_j\\
    [\hat{\mnu}]_{ij} = \frac{1}{|\dE_\dT|}\sum_{e_{u'v'} \in \dE_\dT}[g(h_{u'}^{(L)})]_i \times [g(h_{v'}^{(L)})]_j. 
\end{gather*}

Then, $\mw$ can be solved via:
\begin{align}
    &\min_{\mw} \quad \lVert \hat{\mSigma} \mw  - \hat{\mnu} \rVert_2, \label{eq:optw_noreg} \quad \text{s.t.} \; \mw \geq 0,\,\text{and} \\ &\,\sum_{i,j}[\mw]_{i,j}\bP_\dS(Y_u=i, Y_v=j | e_{uv} \in \dE_\dS) = 1, \nonumber
\end{align}
where the constraints guarantee a valid target edge type distribution $\bP_\dT(Y_u, Y_v | e_{uv} \in \dE_\dT)$. For undirected graphs, $\mw$ can be symmetric, so we may add an extra constraint $[\mw]_{i,j}=[\mw]_{j,i}$. 
Finally, we calculate $\malpha$ following Eq.~\eqref{eq:calcalpha} with the obtained $\mw$ and compute $\mgamma$ via Eq.~\eqref{eq:calcgamma}. Note that in Appendix~\ref{sec:robust}, we will discuss how to improve the robustness of the estimations of $\mw$ and $\mgamma$.

In summary, handling CSS is an iterative process where we begin by employing an estimated $\mgamma$ as edge weights on the source graph to reduce the gap between $\bP_\dS(H^{(L)}|Y)$ and $\bP_\dT(H^{(L)}|Y)$ due to Thm~\ref{thm:decompose}.
 With a reduced gap, we can estimate $\mw$ more accurately (due to Lemma~\ref{lem:optw}) and thus improve the estimation of $\mgamma$. Through iterative refinement, $\mgamma$ progressively enhances the conditional alignment $\bP_\dS(H^{(L)}|Y)=\bP_\dT(H^{(L)}|Y)$ to address CSS. 

\subsection{Addressing Label Shift}
\label{sec:LS}
Inspired by the techniques in~\citet{lipton2018detecting, azizzadenesheli2018regularized}, we estimate the ratio between the source and target label distribution by aligning the node-level pseudo-label distribution to address LS. 
\begin{definition}
    Assume $\bP_\dS(Y_u = i) > 0, \forall i \in \dY$, we define $\mbeta \in \bR^{|\dY|\times 1}$ as the weights of the source and target label distribution:
        $[\mbeta]_{i} = \frac{\bP_\dT(Y = i)}{\bP_\dS(Y = i)}, \forall i\in \dY$.
\end{definition}

\begin{definition}
    Let $\mC \in \bR^{|\dY| \times |\dY|}$ denote the confusion matrix of the classifier for the source domain, and $\mmu\in \bR^{|\dY| \times 1}$ denote the distribution of the label predictions for the target domain, $\forall i, i' \in \dY$ 
    \begin{align*}
        &[\mC]_{i,i'} = \bP_\dS(\hY = i, Y = i'), \quad
        [\mmu]_{i} = \bP_\dT(\hY = i)
    \end{align*}
\end{definition}

The key insight is similar to the estimation of $\mw$, when $\bP_\dS(H^{(L)}|Y) = \bP_\dT(H^{(L)}|Y)$ is satisfied, $\mbeta$ can be estimated by solving a linear system $\mmu = \mC\mbeta$, 

\begin{restatable}{lemma}{betalemma}
    If $\bP_\dS(H^{(L)}|Y) = \bP_\dT(H^{(L)}|Y)$ is satisfied, and node representations are conditionally independent of each other given the node labels, then $\mmu = \mC\mbeta$. 
    \label{lem:optbeta}
\end{restatable}

Empirically, with $\hat{\mC}$ and $\hat{\mmu}$ can be estimated as 
\begin{gather*}
    [\hat{\mC}]_{i,i'} = \frac{1}{|\dV_\dS|}\sum_{u\in \dV_\dS, y_u = i'}[g(h_u^{(L)})]_i\\
    [\hat{\mmu}]_{i} = \frac{1}{|\dV_\dT|}\sum_{u' \in \dV_\dT}[g(h_{u'}^{(L)})]_i
\end{gather*}
\vspace{-4mm}

$\mbeta$ can be solved with a least square problem with the constraints to guarantee a valid target label distribution $\bP_\dT(Y)$. 
\begin{align}
    \min_{\mbeta} \lVert \hat{\mC} \mbeta  - \hat{\mmu} \rVert_2, \,
    \text{s.t.} \; \mbeta \geq 0,\, \sum_{i}[\mbeta]_i\bP_\dS(Y=i) = 1
\label{eq:optbeta_noreg}
\end{align}
We use $\mbeta$ to weight the classification loss to handle LS. Combined with the previous module that uses $\mgamma$ to solve for CSS, our algorithm completely addresses the structure shift.

\subsection{Algorithm Overview}

Now, we are able to put everything together. The entire algorithm is shown in Alg.~\ref{method:alg_table}. At the start of each epoch, the estimated $\mgamma$ are used as edge weights in the source graph (line 4). Then, GNN $\phi_{\mgamma}$ paired with $\mgamma$ yields node representations that further pass through the classifier $g$ to get soft labels $\hat{\mY}$ (line 5). The model is trained via the loss $\mathcal{L}_C^{\mbeta}$, i.e., a $\mbeta$-weighted cross-entropy loss (line 6): 
\begin{align}
    \mathcal{L}_C^{\mbeta} = \frac{1}{|\dV_\dS|}\sum_{v \in \dV_\dS}[\mbeta]_{y_v}\text{cross-entropy}(y_v, \hat{y}_v)
\label{eq:loss}
\end{align}
Then, with every $t$ epoch, update the estimations of $\mw$, $\mgamma$, and $\mbeta$ for the next epoch (lines 7-10). 

\begin{algorithm}[t]
\caption{Pairwise Alignment}
\label{method:alg_table}
\begin{algorithmic}[1]
\STATE \textbf{Input} The source graph $\mathcal{G}_\dS$ with node labels $\mY_\dS$; The target graph $\mathcal{G}_\dT$; A GNN $\phi$ and a classifier $g$; The total epoch number $n$, the epoch period $t$ for weight update. 
\STATE Initialize $\mw, \mgamma, \mbeta = \mathbf{1},$
\WHILE {epoch $< n$ \text{or} \,\text{not converged}} 
    \STATE Add edge weights to $\mathcal{G}_\dS$ according to $\mgamma$
    \STATE  Get $\hat{\mY}_\dS=g(\phi_{\mgamma}(\mx_\dS, \mA_\dS))$ in the source domain
    \STATE  Update $\phi$ and $g$ as $\min_{\phi, g}\mathcal{L}_C^{\mbeta}(\hat{\mY}_\dS,\mY_\dS)$ Eq.~\eqref{eq:loss}
    \IF {$  \text{epoch} \equiv 0\,(\text{mod}\,t)$} 
        \STATE Get $\hat{\mY}_\dS$ and $\hat{\mY}_\dT=g(\phi(\mx_\dT, \mA_\dT))$ 
        \STATE Update the estimation of $\hat{\mSigma}, \hat{\mnu}$, $\hat{\mC}, \hat{\mmu}$
        \STATE Optimize for $\mw$ Eq.\eqref{eq:optw_noreg} and calculate for $\mgamma$ Eq.\eqref{eq:calcgamma}
        \STATE Optimize for $\mbeta$ following Eq.\eqref{eq:optbeta_noreg}
    \ENDIF
\ENDWHILE
\end{algorithmic}
\end{algorithm}
\setlength{\textfloatsep}{10pt}

\subsection{Robust Estimation of $\mgamma,\mw,\mbeta$}
\label{sec:robust}
To improve robustness of the estimation, we incorporate L2 regularization into the least square optimization for $\mw$ and $\mbeta$. Typically, node classification tends to have imperfect accuracy and results in similar prediction probabilities across classes. This may lead to ill-conditioned $\hat{\mSigma}$ and $\hat{\mC}$ in Eq.\eqref{eq:optw_noreg} and \eqref{eq:optbeta_noreg}, respectively.  Specifically,  Eq.\eqref{eq:optw_noreg} and \eqref{eq:optbeta_noreg} can be revised as
\begin{gather}
    \min_{\mw} \lVert \hat{\mSigma} \mw  - \hat{\mnu} \rVert_2 + \lambda \lVert \mw - \mone \rVert_2, \\ 
    \quad \text{s.t.} \;\; \mw \geq 0,\;\sum_{i,j}[\mw]_{i,j}\bP_\dS(Y_u=i, Y_v=j | e_{uv} \in \dE_\dS) = 1\nonumber
    \label{eq:optw}
\end{gather}
\vspace{-4mm}
\begin{gather}
    \min_{\mbeta} \lVert \hat{\mC}\mbeta  - \hat{\mmu}\rVert_2 + \lambda \lVert \mbeta - \mone \rVert_2\\
    \text{s.t.} \;\; \mbeta \geq 0,\; \sum_{i}[\mbeta]_i\bP_\dS(Y=i) = 1.\nonumber
    \label{eq:optbeta}
\end{gather}
where the added L2 regularization will push estimated $\mw$ and $\mbeta$ close to $\mone$. In practice, we find this regularization to be important in the early training stage and can guide a better weight estimation in the later stage.

We also introduce a regularization strategy to improve the robustness of $\mgamma$. This is to deal with the variance in edge formation that may affect $\bP_\dU(Y_v|Y_u, v\in \dN_u)$ in $\mgamma$ calculation. 

Take a specific example to demonstrate the idea of regularizing $\gamma$. Suppose node labels are binary and suppose we count the numbers of edges of different types in the source graph and obtain  $\hat{\bP}_\dS(Y_u = 0, Y_v = 0| e_{uv} \in \dE_\dS) = 0.001$ and $\hat{\bP}_\dS(Y_u = 0, Y_v = 1| e_{uv} \in \dE_\dS) = 0.0005$. Then without any regularization, based on the estimated edge-type distributions, we obtain $\hat{\bP}_\dS(Y_v = 0|Y_u = 0, v\in \dN_u) = 2/3$ and $\hat{\bP}_\dS(Y_v = 0|Y_u = 0, v\in \dN_u) = 1/3$. 
However, the estimation $\hat{\bP}_\dS(Y_u=i, Y_v=j| e_{uv} \in \dE_\dS)$ may be inaccurate when its value is close to 0. Because in this case, the number of edges of the corresponding type $(i,j)$ is too small in the graph. These edges may be formed based on randomness. 
Conversely, larger observed values like $\hat{\bP}_\dS(Y_u = 0, Y_v = 0| e_{uv} \in \dE_\dS) = 0.2$ and $\hat{\bP}_\dS(Y_u = 0, Y_v = 1| e_{uv} \in \dE_\dS) = 0.1$ are often more reliable. 
To address the issue, we may introduce a regularization term $\delta$ when using $\mw$ to compute $\mgamma$. We compute  $\mw'=\frac{\hat{\bP}_\dT(Y_u = i, Y_v = j| e_{uv} \in \dE_\dS)+ \delta}{\hat{\bP}_\dS(Y_u = i, Y_v = j| e_{uv} \in \dE_\dS)+ \delta}= \frac{[\mw]_{ij}\hat{\bP}_\dS(Y_u = i, Y_v = j| e_{uv} \in \dE_\dS) + \delta}{\hat{\bP}_\dS(Y_u = i, Y_v = j| e_{uv} \in \dE_\dS)+ \delta}$, and replace $\mw$ with $\mw'$ when computing $\mgamma$.

\subsection{Comparison to StruRW~\cite{liu2023structural}}
\label{sec:compstrurw}
The edge weights estimation in StruRW and \proj differ in two major points. First, StruRW computes edge weights as the ratio of the source and target edge connection probabilities. This by definition, if using our notations, corresponds to $\mw$ instead of $\mgamma$ and ignores the effect of $\malpha$. However, Thm~\ref{thm:decompose} shows that using $\mgamma$ is the key to reduce CSS. 
Second, even for the estimation of $\mw$, StruRW suffers from inaccurate estimation. In our notation, StruRW simply assumes that $\bP_\dS(\hY = i|Y = i) = 1, \forall i \in \dY$, i.e., perfect training in the source domain and uses hard pseudo-labels in the target domain to estimate $\mw$. In contrast, our optimization to obtain $\mw$ is more stable. Moreover, StruRW ignores the effect of LS entirely. 
From this perspective, StruRW can be understood as a special case of \proj under the assumption of no LS and perfect prediction in the target graph.
Furthermore, our work is the first to rigorously formulate the idea of conditional alignment in graphs.   

\section{Experiments}
\label{sec:exp}

\begin{table*}[t]
\vspace{-3mm}
\caption{Performance on MAG datasets (accuracy scores). The \textbf{bold} font and \underline{underline} indicate the best model and baseline respectively}
\vspace{-4mm}
\begin{center}
\begin{adjustbox}{width = 1\textwidth}
\begin{small}
\begin{sc}
\begin{tabular}{lcccccccccc}
\toprule
 Domains  &   $US \rightarrow CN$ &  $US\rightarrow DE$  &     $US\rightarrow JP$    & $US\rightarrow RU$     & $US\rightarrow FR$ & $CN \rightarrow US$ &  $CN\rightarrow DE$  &     $CN\rightarrow JP$  & $CN\rightarrow RU$     & $CN\rightarrow FR$  \\
\midrule
ERM & $26.92\pm1.08$ &$26.37\pm1.16$ & $37.63\pm0.36$ & $21.71\pm0.38$ & $20.11\pm0.34$  & $31.47\pm1.25$ & $13.29\pm0.36$ & $22.15\pm0.89$ & $10.92\pm0.82$ & $10.86\pm1.04$\\
DANN  & $24.20\pm1.19$ &$26.29\pm1.44$ & $\underline{37.92}\pm0.25$ & $21.76\pm1.58$ & $20.71\pm0.29$  & $30.23\pm0.99$ & $13.46\pm0.40$ & $21.48\pm1.26$ & $11.94\pm1.90$ & $10.65\pm0.53$\\
IWDAN    & $23.39\pm0.93$ &$25.97\pm0.41$ & $34.98\pm0.68$ & $22.80\pm3.03$ & $21.75\pm0.81$  & $31.72\pm1.24$ & $13.39\pm1.06$ & $19.86\pm1.21$ & $10.93\pm1.33$ & $11.64\pm4.56$\\
UDAGCN  & OOM & OOM & OOM & OOM & OOM  & OOM & OOM & OOM & OOM & OOM\\
StruRW  & $\underline{31.58}\pm3.10$ &$\underline{30.03}\pm2.23$ & $37.20\pm0.27$ & $\underline{28.97}\pm2.98$ & $\underline{22.73}\pm1.73$  & $\underline{37.08}\pm1.09$ & $\underline{19.93}\pm1.82$ & $\underline{29.76}\pm2.56$ & $\underline{17.94}\pm9.82$ & $\underline{15.81}\pm3.76$\\
SpecReg  & $23.74\pm1.32$ &$26.68\pm1.44$ & $37.68\pm0.25$ & $21.47\pm0.84$ & $20.91\pm0.53$ & $26.52\pm1.75$ & $13.76\pm0.65$ & $20.50\pm0.08$ & $10.50\pm0.53$ & $10.45\pm1.16$\\
\midrule
\projew & $37.93\pm1.65$ & $38.49\pm2.66$ &$47.38\pm0.61$ & $35.07\pm10.2$  &  $\mathbf{28.64}\pm0.08$ & $43.28\pm0.16$ & $25.91\pm2.70$ & $37.42\pm5.64$ & $32.05\pm0.81$ & $22.83\pm2.46$\\
\projlw & $27.00\pm0.50$ & $26.89\pm0.90$ &$38.96\pm0.94$ & $21.42\pm0.91$  &  $20.63\pm0.45$ & $31.21\pm1.45$ & $15.02\pm1.04$ & $23.22\pm0.57$ & $11.44\pm0.57$ & $11.16\pm0.56$\\
\projb & $\mathbf{40.06}\pm 0.99$ & $\mathbf{38.85}\pm4.71$ & $\mathbf{47.43}\pm 1.82$ & $\mathbf{37.07}\pm5.28$ & $25.21\pm3.79$ & $\mathbf{45.16}\pm0.50$ & $\mathbf{26.19}\pm1.01$ & $\mathbf{38.26}\pm2.27$ & $\mathbf{33.34}\pm1.94$ & $\mathbf{24.16}\pm1.13$\\

\bottomrule
\label{table:MAG}
\end{tabular}
\end{sc}
\end{small}

\end{adjustbox}
\end{center}
\vspace{-7mm}
\end{table*}

\begin{table*}[t]
\vspace{-3mm}
\caption{Performance on Pileup datasets (f1 scores). The \textbf{bold} font and \underline{underline} indicate the best model and baseline respectively}
\vspace{-1mm}
\begin{center}
\begin{adjustbox}{width = 0.85\textwidth}
\begin{small}
\begin{sc}
\begin{tabular}{lcccccccc}
\toprule
 & \multicolumn{6}{c}{Pileup Conditions} & \multicolumn{2}{c}{Physical Processes} \\
 Domains  & $\text{PU}10 \rightarrow 30$    & $\text{PU}30 \rightarrow 10$ & PU$10\rightarrow 50$ & PU$50\rightarrow 10$ & PU$30\rightarrow 140$ & PU$140\rightarrow 30$ &$gg \rightarrow qq$      &     $qq \rightarrow gg$      \\
\midrule
ERM & $48.17\pm3.87$ &$64.17\pm1.50$ & $48.73\pm0.45$ & $70.11\pm1.12$ & $18.76\pm1.50$  & $33.02\pm28.77$ & $\underline{67.70}\pm0.31$ & $72.63\pm0.54$\\
DANN  & $49.99\pm2.07$ &$64.62\pm0.70$ & $48.44\pm0.78$ & $68.70\pm1.42$ & $28.20\pm1.20$  & $21.95\pm20.37$ & $66.48\pm0.67$ & $71.78\pm0.87$\\
IWDAN    & $35.85\pm1.73$ &$62.24\pm0.15$ & $26.49\pm0.40$ & $67.82\pm0.62$ & $8.91\pm3.17$  & $\underline{40.02}\pm1.93$ & $66.85\pm0.69$ & $\underline{73.10}\pm0.29$\\
UDAGCN  &$45.39\pm2.07$ &$62.27\pm1.23$ & $44.75\pm1.76$ & $68.93\pm0.55$ & $19.95\pm0.84$  & $29.66\pm5.57$ & $65.99\pm1.06$ & $71.99\pm0.61$\\
StruRW   & $52.41\pm1.74$ &$\underline{67.72}\pm0.22$ & $47.25\pm1.96$ & $\underline{70.93}\pm0.66$ & $\underline{37.81}\pm0.64$  & $37.84\pm2.82$ & $67.66\pm0.55$ & $72.72\pm0.68$\\
SpecReg  & $\underline{52.61}\pm1.06$ &$65.34\pm0.62$ & $\underline{48.85}\pm0.94$ & $67.95\pm2.23$ & $28.86\pm1.58$ & $28.79\pm25.83$ & $66.66\pm0.40$ & $72.73\pm0.42$\\
\midrule
  \projew & $\mathbf{56.00}\pm0.14$ & $58.44\pm3.19$ & $50.77\pm 0.70$ &$60.95\pm6.09$ & $40.31\pm0.31$  &  $37.24\pm7.69$ & $67.75\pm0.27$ & $73.24\pm0.38$\\
\projlw & $46.84\pm0.45$ & $67.12\pm0.65$ &$48.51\pm1.46$ & $71.17\pm0.70$  &  $36.29\pm0.92$ & $46.38\pm0.96$ & $67.63\pm0.38$ & $\mathbf{73.40}\pm0.13$\\
\projb & $55.45\pm 0.21$ & $\mathbf{68.29}\pm0.41$ & $\mathbf{51.43}\pm 0.42$ & $\mathbf{71.23}\pm0.63$ & $\mathbf{40.53}\pm0.25$& $\mathbf{51.21}\pm2.88$& $\mathbf{67.77}\pm0.70$ & $73.36\pm0.12$\\

\bottomrule
\label{table:hep}
\end{tabular}
\end{sc}
\end{small}

\end{adjustbox}
\end{center}
\vspace{-8mm}
\end{table*}

\begin{table*}[t]
\vspace{-4mm}
\caption{Synthetic CSBM results (accuracy). The \textbf{bold} font and the \underline{underline} indicate the best model and baseline respectively}
\vspace{-1mm}
\begin{center}
\begin{adjustbox}{width=0.89\textwidth}
\begin{small}
\begin{sc}
\begin{tabular}{lcccccccc}
\toprule
 & \multicolumn{2}{c}{CSS (only class ratio shift)} & \multicolumn{2}{c}{CSS (only degree shift)}  & \multicolumn{2}{c}{CSS (shift in both)} & \multicolumn{2}{c}{CSS + LS}\\

\midrule
ERM & $94.22\pm0.97$ & $57.04\pm3.83$ & $99.01\pm0.28$ & $96.21\pm0.27$ & $88.90\pm0.22$ & $58.01\pm1.91$ & $61.35\pm4.64$ & $61.65\pm0.80$\\
IWDAN    & $95.85\pm0.70$ & $76.75\pm1.32$ & $98.97\pm0.05$ & $\textbf{97.15}\pm0.33$ & $\underline{93.65}\pm0.70$ & $79.53\pm3.57$ & $\underline{92.42}\pm0.72$ & $\underline{87.01}\pm2.14$\\
UDAGCN   & $96.82\pm0.70$ & $69.93\pm5.17$ & $\textbf{99.52}\pm0.05$ & $\underline{97.04}\pm0.28$ & $93.17\pm1.02$ & $67.44\pm4.95$ & $87.67\pm3.21$ & $83.69\pm2.35$\\
StruRW   & $\underline{96.83}\pm0.33$ & $\underline{86.65}\pm5.62$ & $98.87\pm0.19$ & $95.93\pm0.55$ & $92.09\pm0.55$ & $\underline{80.00}\pm7.49$ & $75.38\pm12.11$ & $75.96\pm2.96$\\
SpecReg   & $93.46\pm1.21$ & $62.97\pm1.01$ & $98.94\pm0.03$ & $96.69\pm0.23$ & $89.58\pm1.58$ & $61.28\pm1.19$ & $76.73\pm3.18$ & $83.40\pm1.38$\\
\midrule
\projew  & $96.65\pm1.21$ & $91.79\pm1.68$ & $98.92\pm0.52$ & $96.24\pm0.23$ & $94.99\pm0.49$ & $91.20\pm0.95$ & $94.95\pm0.69$ & $\mathbf{95.66}\pm0.45$\\
\projlw & $94.22\pm0.95$ & $57.14\pm3.73$ &$\underline{99.02}\pm0.29$ & $96.17\pm0.26$  &  $88.85\pm0.22$ & $57.96\pm1.84$ & $61.39\pm4.59$ & $67.91\pm9.98$\\
\projb    & $\mathbf{97.24}\pm0.33$ & $\mathbf{91.97}\pm1.49$ & $98.20\pm1.04$ & $96.25\pm0.33$ & $\mathbf{95.44}\pm0.51$ & $\mathbf{91.67}\pm0.38$ & $\mathbf{95.24}\pm0.11$ & $95.55\pm0.65$\\
\bottomrule
\label{table:CSBM}
\end{tabular}
\end{sc}
\end{small}
\end{adjustbox}
\end{center}
\vspace{-7mm}
\end{table*}

\begin{table*}[t]
\vspace{-2.5mm}
\caption{Performance on Arxiv and DBLP/ACM datasets (accuracy). The \textbf{bold} and \underline{underline} indicate the best model and baseline}
\vspace{-1mm}
\begin{center}
\begin{adjustbox}{width = 0.85\textwidth}
\begin{small}
\begin{sc}
\begin{tabular}{lcccccccc}
\toprule
 & \multicolumn{2}{c}{1950-2007} & \multicolumn{2}{c}{1950-2009} & \multicolumn{2}{c}{1950-2011} & \multicolumn{2}{c}{DBLP and ACM}\\
 Domains  & $2014-2016$    & $2016-2018$ & $2014-2016$ & $2016-2018$ & $2014-2016$ & $2016-2018$  & $A\rightarrow D$ & $D\rightarrow A$\\
\midrule
ERM & $37.91\pm0.31$ &$35.22\pm0.71$ & $43.50\pm0.35$ & $40.19\pm3.62$ & $51.76\pm0.93$  & $52.56\pm1.06$ & $57.26\pm1.90$  & $47.77\pm6.61$\\
DANN  & $37.31\pm1.54$ &$36.84\pm1.40$ & $\underline{43.57}\pm0.47$ & $42.04\pm2.70$ & $53.02\pm0.67$  & $52.69\pm1.26$ & $65.34\pm5.91$  & $54.36\pm6.20$\\
IWDAN  & $36.16\pm2.91$ &$25.48\pm9.77$ & $41.26\pm2.08$ & $35.91\pm4.28$ & $46.73\pm0.62$  & $42.70\pm3.21$ & $\underline{66.96}\pm7.38$  & $56.13\pm6.48$\\
UDAGCN  &$38.10\pm1.62$ &OOM & $42.85\pm2.09$ & OOM & $53.13\pm0.31$  & OOM & $57.05\pm5.43$  & $\underline{58.42}\pm6.65$\\
StruRW  & $\underline{38.56}\pm0.77$ &$\underline{37.17}\pm2.75$ & $43.55\pm2.37$ & $\underline{43.55}\pm2.37$ & $\underline{53.19}\pm0.45$  & $\mathbf{53.64}\pm0.65$ & $60.03\pm2.18$  & $52.13\pm1.25$\\
SpecReg & $37.09\pm0.62$ &$33.46\pm0.83$ & $43.14\pm2.16$ & $43.06\pm1.09$ & $52.63\pm1.29$ & $52.46\pm0.83$ & $31.03\pm2.45$  & $53.04\pm2.21$\\
\midrule
\projew & $39.75\pm0.96$ &$40.54\pm2.44$ & $44.04\pm0.83$ & $44.32\pm1.61$ & $\mathbf{53.75}\pm0.48$ & $51.10\pm1.30$  & $65.20\pm3.69$  & $60.60\pm3.86$\\
\projlw & $39.47\pm0.88$ & $\mathbf{41.14}\pm2.07$ &$43.40\pm1.97$ & $43.44\pm1.65$  &  $52.48\pm0.53$ & $\underline{52.83}\pm0.98$ & $\mathbf{72.41}\pm1.29$ & $61.40\pm1.92$\\
\projb & $\mathbf{39.98}\pm0.77$ &$40.23\pm0.30$ & $\mathbf{44.60}\pm0.42$ & $\mathbf{44.43} \pm0.34$ & $53.56\pm0.98$  & $51.60\pm0.24$ & $70.97\pm3.87$  & $\mathbf{63.36}\pm2.90$\\

\bottomrule
\label{table:citation}
\end{tabular}
\end{sc}
\end{small}
\end{adjustbox}
\end{center}
\vspace{-7mm}
\end{table*}

We evaluate three variants of \proj to understand how its different components deal with the distribution shift on synthetic datasets and 5 real-world datasets. These variants include \projew with only $\mgamma$ as source graph edge weights to address CSS, \projlw with only $\mbeta$ as label weights to address LS, and  \projb that combines both. 
We next briefly introduce datasets and settings while leaving more details in Appendix~\ref{app:exp}.

\subsection{Datasets and Experimental Settings} 
\textbf{Synthetic Data.}
CSBMs (see the definition in Appendix~\ref{app:CSBM}) are used to generate the source and target graphs with three node classes. We explore four scenarios in structure shift without feature shift, where the first three explore CSS with shifts in the conditional neighboring node's label distribution (class ratio), shifts in the conditional node's degree distribution (degree), and shifts in both. Considering these three types of shift is inspired by the argument in Thm~\ref{thm:decompose}. The fourth setting examines CSS and LS jointly. 
In addition, we consider two degrees of shift under each scenario with the left column being the smaller shift as shown in Table~\ref{table:CSBM}. 
The detailed configurations of the CSBM regarding edge probabilities and node features are in Appendix~\ref{app:expset}. 

\textbf{MAG} We extract paper nodes and their citation links from the original MAG~\cite{hu2020open, wang2020microsoft}. Papers are split into separate graphs based on their countries of publication determined by their corresponding authors. The task is to classify the publication venue of the papers. Our experiments study generation across the top 6 countries with the most number of papers (in total 377k nodes, 1.35M edges). We train models on the graphs from US/China and test them on the graphs from the rest countries.

\textbf{Pileup Mitigation}~\cite{liu2023structural} is a dataset of a denoising task in HEP named pileup mitigation~\cite{bertolini2014pileup}. Proton-proton collisions produce particles with leading collisions (LC) and nearby bunch crossings as other collisions (OC). The task is to identify whether a particle is from LC or OC. Nodes are particles and particles are connected if they are close in the $\eta$-$\phi$ space. We study two distribution shifts: the shift of pile-up (PU) conditions (mostly structure shift), where PU$k$ indicates the averaged number of other collisions in the beam is $k$, and the shift in the data generating process (primarily feature shift).

\textbf{Arxiv}~\cite{hu2020open} is a citation network of Arxiv papers to classify papers' subject areas. We study the shift in time by using papers published in earlier periods to train and test on papers published later. Specifically, we traine on papers published from 1950 to 2007/ 2009/ 2011 and test on paper published between 2014 to 2016 and 2016 to 2018.

\textbf{DBLP and ACM}~\cite{tang2008arnetminer, wu2020unsupervised} are two paper citation networks obtained from DBLP and ACM. Nodes are papers and edges represent citations between papers. The goal is to predict the research topic of a paper. We train the GNN on one network and test it on the other.

\textbf{Baselines} DANN~\cite{ganin2016domain} and IWDAN~\cite{tachet2020domain} are non-graph methods, we adapt them to the graph setting with GNN as the encoder. UDAGCN~\cite{wu2020unsupervised}, StruRW~\cite{liu2023structural} and SpecReg~\cite{you2023graph} are chosen as GDA baselines. We use GraphSAGE~\cite{hamilton2017inductive} as backbones and the same model architecture for all baselines. 

\textbf{Evaluation and Metric}
The source graph is used for training, 20 percent of the node labels in the target graph are used for validation and the rest 80 percent are held out for testing. We select the best model based on the target validation scores and report its scores on the target testing nodes in tables. We use accuracy for MAG, Arxiv, DBLP, ACM, and synthetic datasets. For the MAG datasets, we evaluate the top 19 classes as we group the remaining classes as a dummy class. The Pileup dataset uses the binary f1 score. 

\textbf{Hyperparameter Study}
Our hyperparameter tuning is mainly for the robustness estimation for $\mgamma$ and $\mbeta$ detailed in section~\ref{sec:robust}. We will discuss them in Appendix~\ref{app:hyperparameter}.

\subsection{Result Analysis}
In the MAG dataset, \proj methods markedly outperform baselines, as detailed in Table~\ref{table:MAG}. Most baselines generally match the performance of ERM suggesting their limited effectiveness in addressing CSS and LS. StruRW, however, stands out, emphasizing the need for CSS mitigation in MAG. When compared to StruRW, \proj not only demonstrates superior handling of CSS but also offers advantages in LS mitigation, resulting in over $25\%$ relative improvements. Also, IWDAN has not shown improvements due to the suboptimality of performing only conditional feature alignment yet ignoring the structure, 
highlighting the importance of tailored solutions for GDA like \proj. 

HEP results are in Table~\ref{table:hep}. 
Considering the shift in pileup (PU) conditions, baselines with graph structure regularization, like StruRW and SpecReg, achieve better performance. This matches our expectations that PU condition shifts introduce mostly structure shifts as shown in Fig~\ref{fig:shift} and our methods further significantly outperform these baselines in addressing such shifts. Specifically, we observe \projew excels in transitioning from low PU to high PU, while \projlw is more effective in the opposite direction. This difference stems from the varying dominant impacts of LS and CSS. High PU datasets have more imbalanced label distribution with a large OC: LC ratio, 
where LS induces more negative effects over CSS, necessitating the LS mitigation. Conversely, the cases from low PU to high PU, mainly influenced by CSS, can be addressed better by \projew. Regarding shifts in physical processes, \proj methods still rank the best, but all models have close performance since structure shift now becomes minor 
as shown in Table~\ref{table:hepstats}.

The synthetic dataset results in Table~\ref{table:CSBM} well justify our theory. We observe minimal performance decay with ERM in scenarios with only degree shifts, indicating that node degree impacts are minimal under mean pooling in GNNs. Additionally, while CSS with both shifts results in lower ERM performance compared to shift only in class ratio, our \proj method achieves similar performance, highlighting the adequacy of focusing on shifts in the conditional neighborhood node label distribution for CSS. \proj notably outperforms baselines in CSS scenarios, especially where class ratio shifts are more pronounced (as in the second case of each scenario). With joint shifts in CSS and LS, \proj methods perform the best and IWDAN is the best baseline as it is designed to address conditional shifts and LS in non-graph tasks.

For the Arxiv and DBLP/ACM datasets in Table~\ref{table:citation}, the \proj methods demonstrate reasonable improvements over baselines. Regarding the Arxiv dataset, \proj is particularly effective when the training on pre-2007 papers, which possess larger shifts as shown in Table~\ref{table:realstats}. Also, all baselines perform similarly with no significant gap between the GDA methods and the non-graph methods, suggesting that addressing structure shift has limited benefits in this dataset. Likewise, regarding the DBLP and ACM datasets, we observe the performance gain of methods that align marginal node feature distribution, like DANN and UDAGCN, indicating this dataset contains mostly feature shifts. While in the cases where LS is large ($A\rightarrow D$ or Arxiv training on pre-2007, testing on 2016-2018 as shown in Table~\ref{table:realstats}), \projlw achieves the best performance. 

\vspace{1mm}
\textbf{Ablation Study}

Among the three variants of \proj, \projb performs the best in most cases. \projew contributes more compared to \projlw when CSS dominates (MAG datasets, Arxiv, and HEP from low PU to high PU). \projlw alone offers slight improvements except with highly imbalanced training labels (from high PU to low PU in HEP datasets). But when combined with \projew, it will yield additional benefits.

\section{Conclusion}
This work studies the distribution shifts in graph-structured data. 
We analyze distribution shifts in real-world graph data and decompose structure shifts into two components: conditional structure shift (CSS) and label shift (LS). 
Our novel approach, Pairwise Alignment (\proj), well tackles both CSS and LS in both theory and practice. Importantly, this work also curates a new, by far the largest dataset MAG which reflects the actual need for region-based generalization of graph learning models. We believe this large dataset can incentivize more in-depth studies on GDA.  
\section*{Impact Statement}
This paper presents work whose goal is to advance the field of Machine Learning. There are many potential societal consequences of our work, none which we feel must be specifically highlighted here.

\section*{Acknowledgement}
We greatly thank Yongbin Feng for discussing relevant HEP applications and Mufei Li for discussing relevant MAG dataset curation. S. Liu, D. Zou, and P. Li are partially supported by NSF award PHY-2117997 and IIS-2239565. The work of HZ was supported in part by the Defense Advanced Research Projects Agency (DARPA) under Cooperative Agreement Number: HR00112320012 and a research grant from the IBM-Illinois Discovery Accelerator Institute (IIDAI).

\newpage
\bibliography{reference}

\begin{thebibliography}{77}
\providecommand{\natexlab}[1]{#1}
\providecommand{\url}[1]{\texttt{#1}}
\expandafter\ifx\csname urlstyle\endcsname\relax
  \providecommand{\doi}[1]{doi: #1}\else
  \providecommand{\doi}{doi: \begingroup \urlstyle{rm}\Url}\fi

\bibitem[Azizzadenesheli et~al.(2018)Azizzadenesheli, Liu, Yang, and Anandkumar]{azizzadenesheli2018regularized}
Azizzadenesheli, K., Liu, A., Yang, F., and Anandkumar, A.
\newblock Regularized learning for domain adaptation under label shifts.
\newblock \emph{International Conference on Learning Representations}, 2018.

\bibitem[Bertolini et~al.(2014)Bertolini, Harris, Low, and Tran]{bertolini2014pileup}
Bertolini, D., Harris, P., Low, M., and Tran, N.
\newblock Pileup per particle identification.
\newblock \emph{Journal of High Energy Physics}, 2014.

\bibitem[Bevilacqua et~al.(2021)Bevilacqua, Zhou, and Ribeiro]{bevilacqua2021size}
Bevilacqua, B., Zhou, Y., and Ribeiro, B.
\newblock Size-invariant graph representations for graph classification extrapolations.
\newblock \emph{International Conference on Machine Learning}, 2021.

\bibitem[Cai et~al.(2021)Cai, Wu, Li, Wei, Yi, and Zhang]{cai2021graph}
Cai, R., Wu, F., Li, Z., Wei, P., Yi, L., and Zhang, K.
\newblock Graph domain adaptation: A generative view.
\newblock \emph{arXiv preprint arXiv:2106.07482}, 2021.

\bibitem[Chen et~al.(2022)Chen, Zhang, Bian, Yang, Kaili, Xie, Liu, Han, and Cheng]{chen2022learning}
Chen, Y., Zhang, Y., Bian, Y., Yang, H., Kaili, M., Xie, B., Liu, T., Han, B., and Cheng, J.
\newblock Learning causally invariant representations for out-of-distribution generalization on graphs.
\newblock \emph{Advances in Neural Information Processing Systems}, 2022.

\bibitem[Chen et~al.(2023)Chen, Bian, Zhou, Xie, Han, and Cheng]{chen2023does}
Chen, Y., Bian, Y., Zhou, K., Xie, B., Han, B., and Cheng, J.
\newblock Does invariant graph learning via environment augmentation learn invariance?
\newblock \emph{Advances in Neural Information Processing Systems}, 2023.

\bibitem[Chuang \& Jegelka(2022)Chuang and Jegelka]{chuang2022tree}
Chuang, C.-Y. and Jegelka, S.
\newblock Tree mover's distance: Bridging graph metrics and stability of graph neural networks.
\newblock \emph{Advances in Neural Information Processing Systems}, 2022.

\bibitem[Cicek \& Soatto(2019)Cicek and Soatto]{cicek2019unsupervised}
Cicek, S. and Soatto, S.
\newblock Unsupervised domain adaptation via regularized conditional alignment.
\newblock \emph{Proceedings of the IEEE/CVF international conference on computer vision}, 2019.

\bibitem[Dai et~al.(2022)Dai, Wu, Xiao, Shen, and Wang]{dai2022graph}
Dai, Q., Wu, X.-M., Xiao, J., Shen, X., and Wang, D.
\newblock Graph transfer learning via adversarial domain adaptation with graph convolution.
\newblock \emph{IEEE Transactions on Knowledge and Data Engineering}, 2022.

\bibitem[Deshpande et~al.(2018)Deshpande, Sen, Montanari, and Mossel]{deshpande2018contextual}
Deshpande, Y., Sen, S., Montanari, A., and Mossel, E.
\newblock Contextual stochastic block models.
\newblock \emph{Advances in Neural Information Processing Systems}, 31, 2018.

\bibitem[Ding et~al.(2021)Ding, Kong, Chen, Kirchenbauer, Goldblum, Wipf, Huang, and Goldstein]{ding2021closer}
Ding, M., Kong, K., Chen, J., Kirchenbauer, J., Goldblum, M., Wipf, D., Huang, F., and Goldstein, T.
\newblock A closer look at distribution shifts and out-of-distribution generalization on graphs.
\newblock \emph{NeurIPS 2021 Workshop on Distribution Shifts: Connecting Methods and Applications}, 2021.

\bibitem[Dou et~al.(2020)Dou, Liu, Sun, Deng, Peng, and Yu]{dou2020enhancing}
Dou, Y., Liu, Z., Sun, L., Deng, Y., Peng, H., and Yu, P.~S.
\newblock Enhancing graph neural network-based fraud detectors against camouflaged fraudsters.
\newblock \emph{Proceedings of the 29th ACM international conference on information \& knowledge management}, 2020.

\bibitem[Fan et~al.(2022)Fan, Wang, Mo, Shi, and Tang]{fan2022debiasing}
Fan, S., Wang, X., Mo, Y., Shi, C., and Tang, J.
\newblock Debiasing graph neural networks via learning disentangled causal substructure.
\newblock \emph{Advances in Neural Information Processing Systems}, 2022.

\bibitem[Fan et~al.(2023)Fan, Wang, Shi, Cui, and Wang]{fan2023generalizing}
Fan, S., Wang, X., Shi, C., Cui, P., and Wang, B.
\newblock Generalizing graph neural networks on out-of-distribution graphs.
\newblock \emph{IEEE Transactions on Pattern Analysis and Machine Intelligence}, 2023.

\bibitem[Ganin et~al.(2016)Ganin, Ustinova, Ajakan, Germain, Larochelle, Laviolette, Marchand, and Lempitsky]{ganin2016domain}
Ganin, Y., Ustinova, E., Ajakan, H., Germain, P., Larochelle, H., Laviolette, F., Marchand, M., and Lempitsky, V.
\newblock Domain-adversarial training of neural networks.
\newblock \emph{The journal of machine learning research}, 2016.

\bibitem[Gong et~al.(2016)Gong, Zhang, Liu, Tao, Glymour, and Sch{\"o}lkopf]{gong2016domain}
Gong, M., Zhang, K., Liu, T., Tao, D., Glymour, C., and Sch{\"o}lkopf, B.
\newblock Domain adaptation with conditional transferable components.
\newblock \emph{International Conference on Machine Learning}, 2016.

\bibitem[Gui et~al.(2023)Gui, Liu, Li, Luo, and Ji]{gui2023joint}
Gui, S., Liu, M., Li, X., Luo, Y., and Ji, S.
\newblock Joint learning of label and environment causal independence for graph out-of-distribution generalization.
\newblock \emph{Advances in Neural Information Processing Systems}, 2023.

\bibitem[Hamilton et~al.(2017)Hamilton, Ying, and Leskovec]{hamilton2017inductive}
Hamilton, W., Ying, Z., and Leskovec, J.
\newblock Inductive representation learning on large graphs.
\newblock \emph{Advances in Neural Information Processing Systems}, 2017.

\bibitem[Han et~al.(2022)Han, Jiang, Liu, and Hu]{han2022g}
Han, X., Jiang, Z., Liu, N., and Hu, X.
\newblock G-mixup: Graph data augmentation for graph classification.
\newblock \emph{International Conference on Machine Learning}, 2022.

\bibitem[Highfield(2008)]{highfield2008large}
Highfield, R.
\newblock Large hadron collider: Thirteen ways to change the world.
\newblock \emph{The Daily Telegraph. London. Retrieved}, 2008.

\bibitem[Hu et~al.(2020)Hu, Fey, Zitnik, Dong, Ren, Liu, Catasta, and Leskovec]{hu2020open}
Hu, W., Fey, M., Zitnik, M., Dong, Y., Ren, H., Liu, B., Catasta, M., and Leskovec, J.
\newblock Open graph benchmark: Datasets for machine learning on graphs.
\newblock \emph{Advances in Neural Information Processing Systems}, 2020.

\bibitem[Jackson et~al.(2008)]{jackson2008social}
Jackson, M.~O. et~al.
\newblock \emph{Social and economic networks}, volume~3.
\newblock Princeton university press Princeton, 2008.

\bibitem[Ji et~al.(2023)Ji, Zhang, Wu, Wu, Li, Huang, Xu, Rong, Ren, Xue, et~al.]{ji2023drugood}
Ji, Y., Zhang, L., Wu, J., Wu, B., Li, L., Huang, L.-K., Xu, T., Rong, Y., Ren, J., Xue, D., et~al.
\newblock Drugood: Out-of-distribution dataset curator and benchmark for ai-aided drug discovery--a focus on affinity prediction problems with noise annotations.
\newblock \emph{Proceedings of the AAAI Conference on Artificial Intelligence}, 2023.

\bibitem[Jia et~al.(2023)Jia, Li, Yang, Tao, and Shi]{jia2023graph}
Jia, T., Li, H., Yang, C., Tao, T., and Shi, C.
\newblock Graph invariant learning with subgraph co-mixup for out-of-distribution generalization.
\newblock \emph{arXiv preprint arXiv:2312.10988}, 2023.

\bibitem[Jin et~al.(2022)Jin, Zhao, Ding, Liu, Tang, and Shah]{jin2022empowering}
Jin, W., Zhao, T., Ding, J., Liu, Y., Tang, J., and Shah, N.
\newblock Empowering graph representation learning with test-time graph transformation.
\newblock \emph{International Conference on Learning Representations}, 2022.

\bibitem[Kipf \& Welling(2016)Kipf and Welling]{kipf2016semi}
Kipf, T.~N. and Welling, M.
\newblock Semi-supervised classification with graph convolutional networks.
\newblock \emph{International Conference on Learning Representations}, 2016.

\bibitem[Koh et~al.(2021)Koh, Sagawa, Marklund, Xie, Zhang, Balsubramani, Hu, Yasunaga, Phillips, Gao, et~al.]{koh2021wilds}
Koh, P.~W., Sagawa, S., Marklund, H., Xie, S.~M., Zhang, M., Balsubramani, A., Hu, W., Yasunaga, M., Phillips, R.~L., Gao, I., et~al.
\newblock Wilds: A benchmark of in-the-wild distribution shifts.
\newblock \emph{International Conference on Machine Learning}, 2021.

\bibitem[Komiske et~al.(2017)Komiske, Metodiev, Nachman, and Schwartz]{komiske2017pileup}
Komiske, P.~T., Metodiev, E.~M., Nachman, B., and Schwartz, M.~D.
\newblock Pileup mitigation with machine learning (pumml).
\newblock \emph{Journal of High Energy Physics}, 2017.

\bibitem[Li et~al.(2022{\natexlab{a}})Li, Zhang, Wang, and Zhu]{li2022learning}
Li, H., Zhang, Z., Wang, X., and Zhu, W.
\newblock Learning invariant graph representations for out-of-distribution generalization.
\newblock \emph{Advances in Neural Information Processing Systems}, 2022{\natexlab{a}}.

\bibitem[Li et~al.(2022{\natexlab{b}})Li, Liu, Feng, Paspalaki, Tran, Liu, and Li]{li2022semi}
Li, T., Liu, S., Feng, Y., Paspalaki, G., Tran, N., Liu, M., and Li, P.
\newblock Semi-supervised graph neural networks for pileup noise removal.
\newblock \emph{The European Physics Journal C}, 2022{\natexlab{b}}.

\bibitem[Liao et~al.(2021)Liao, Zhao, Xu, Jaakkola, Gordon, Jegelka, and Salakhutdinov]{liao2021information}
Liao, P., Zhao, H., Xu, K., Jaakkola, T., Gordon, G.~J., Jegelka, S., and Salakhutdinov, R.
\newblock Information obfuscation of graph neural networks.
\newblock \emph{International Conference on Machine Learning}, 2021.

\bibitem[Ling et~al.(2023)Ling, Jiang, Liu, Ji, and Zou]{ling2023graph}
Ling, H., Jiang, Z., Liu, M., Ji, S., and Zou, N.
\newblock Graph mixup with soft alignments.
\newblock \emph{International Conference on Machine Learning}, 2023.

\bibitem[Lipton et~al.(2018)Lipton, Wang, and Smola]{lipton2018detecting}
Lipton, Z., Wang, Y.-X., and Smola, A.
\newblock Detecting and correcting for label shift with black box predictors.
\newblock \emph{International Conference on Machine Learning}, 2018.

\bibitem[Liu et~al.(2024)Liu, Fang, Zhang, Gu, Zhou, Wang, and Bu]{liu2024rethinking}
Liu, M., Fang, Z., Zhang, Z., Gu, M., Zhou, S., Wang, X., and Bu, J.
\newblock Rethinking propagation for unsupervised graph domain adaptation.
\newblock \emph{arXiv preprint arXiv:2402.05660}, 2024.

\bibitem[Liu et~al.(2023)Liu, Li, Feng, Tran, Zhao, Qiu, and Li]{liu2023structural}
Liu, S., Li, T., Feng, Y., Tran, N., Zhao, H., Qiu, Q., and Li, P.
\newblock Structural re-weighting improves graph domain adaptation.
\newblock \emph{International Conference on Machine Learning}, 2023.

\bibitem[Liu et~al.(2021)Liu, Guo, Li, Xing, You, Kuo, El~Fakhri, and Woo]{liu2021adversarial}
Liu, X., Guo, Z., Li, S., Xing, F., You, J., Kuo, C.-C.~J., El~Fakhri, G., and Woo, J.
\newblock Adversarial unsupervised domain adaptation with conditional and label shift: Infer, align and iterate.
\newblock \emph{Proceedings of the IEEE/CVF international conference on computer vision}, 2021.

\bibitem[Long et~al.(2015)Long, Cao, Wang, and Jordan]{long2015learning}
Long, M., Cao, Y., Wang, J., and Jordan, M.
\newblock Learning transferable features with deep adaptation networks.
\newblock \emph{International Conference on Machine Learning}, 2015.

\bibitem[Long et~al.(2018)Long, Cao, Wang, and Jordan]{long2018conditional}
Long, M., Cao, Z., Wang, J., and Jordan, M.~I.
\newblock Conditional adversarial domain adaptation.
\newblock \emph{Advances in Neural Information Processing Systems}, 2018.

\bibitem[Miao et~al.(2022)Miao, Liu, and Li]{miao2022interpretable}
Miao, S., Liu, M., and Li, P.
\newblock Interpretable and generalizable graph learning via stochastic attention mechanism.
\newblock \emph{International Conference on Machine Learning}, 2022.

\bibitem[Pang et~al.(2023)Pang, Wang, Tang, Xiao, and Yin]{pang2023sa}
Pang, J., Wang, Z., Tang, J., Xiao, M., and Yin, N.
\newblock Sa-gda: Spectral augmentation for graph domain adaptation.
\newblock \emph{Proceedings of the 31st ACM International Conference on Multimedia}, 2023.

\bibitem[Peters et~al.(2017)Peters, Janzing, and Sch{\"o}lkopf]{peters2017elements}
Peters, J., Janzing, D., and Sch{\"o}lkopf, B.
\newblock \emph{Elements of causal inference: foundations and learning algorithms}.
\newblock The MIT Press, 2017.

\bibitem[Rojas-Carulla et~al.(2018)Rojas-Carulla, Sch{\"o}lkopf, Turner, and Peters]{rojas2018invariant}
Rojas-Carulla, M., Sch{\"o}lkopf, B., Turner, R., and Peters, J.
\newblock Invariant models for causal transfer learning.
\newblock \emph{The Journal of Machine Learning Research}, 2018.

\bibitem[Shen et~al.(2020{\natexlab{a}})Shen, Dai, Chung, Lu, and Choi]{shen2020adversarial}
Shen, X., Dai, Q., Chung, F.-l., Lu, W., and Choi, K.-S.
\newblock Adversarial deep network embedding for cross-network node classification.
\newblock \emph{Proceedings of the AAAI conference on artificial intelligence}, 2020{\natexlab{a}}.

\bibitem[Shen et~al.(2020{\natexlab{b}})Shen, Dai, Mao, Chung, and Choi]{shen2020network}
Shen, X., Dai, Q., Mao, S., Chung, F.-l., and Choi, K.-S.
\newblock Network together: Node classification via cross-network deep network embedding.
\newblock \emph{IEEE Transactions on Neural Networks and Learning Systems}, 2020{\natexlab{b}}.

\bibitem[Shervashidze et~al.(2011)Shervashidze, Schweitzer, Van~Leeuwen, Mehlhorn, and Borgwardt]{shervashidze2011weisfeiler}
Shervashidze, N., Schweitzer, P., Van~Leeuwen, E.~J., Mehlhorn, K., and Borgwardt, K.~M.
\newblock Weisfeiler-lehman graph kernels.
\newblock \emph{Journal of Machine Learning Research}, 12\penalty0 (9), 2011.

\bibitem[Shlomi et~al.(2020)Shlomi, Battaglia, and Vlimant]{shlomi2020graph}
Shlomi, J., Battaglia, P., and Vlimant, J.-R.
\newblock Graph neural networks in particle physics.
\newblock \emph{Machine Learning: Science and Technology}, 2020.

\bibitem[Sui et~al.(2023)Sui, Wu, Wu, Cui, Li, Zhou, Wang, and He]{sui2023unleashing}
Sui, Y., Wu, Q., Wu, J., Cui, Q., Li, L., Zhou, J., Wang, X., and He, X.
\newblock Unleashing the power of graph data augmentation on covariate distribution shift.
\newblock \emph{Advances in Neural Information Processing Systems}, 2023.

\bibitem[Szklarczyk et~al.(2019)Szklarczyk, Gable, Lyon, Junge, Wyder, Huerta-Cepas, Simonovic, Doncheva, Morris, Bork, et~al.]{szklarczyk2019string}
Szklarczyk, D., Gable, A.~L., Lyon, D., Junge, A., Wyder, S., Huerta-Cepas, J., Simonovic, M., Doncheva, N.~T., Morris, J.~H., Bork, P., et~al.
\newblock String v11: protein--protein association networks with increased coverage, supporting functional discovery in genome-wide experimental datasets.
\newblock \emph{Nucleic acids research}, 2019.

\bibitem[Tachet~des Combes et~al.(2020)Tachet~des Combes, Zhao, Wang, and Gordon]{tachet2020domain}
Tachet~des Combes, R., Zhao, H., Wang, Y.-X., and Gordon, G.~J.
\newblock Domain adaptation with conditional distribution matching and generalized label shift.
\newblock \emph{Advances in Neural Information Processing Systems}, 2020.

\bibitem[Tang et~al.(2008)Tang, Zhang, Yao, Li, Zhang, and Su]{tang2008arnetminer}
Tang, J., Zhang, J., Yao, L., Li, J., Zhang, L., and Su, Z.
\newblock Arnetminer: extraction and mining of academic social networks.
\newblock \emph{Proceedings of the 14th ACM SIGKDD international conference on Knowledge discovery and data mining}, 2008.

\bibitem[Tzeng et~al.(2017)Tzeng, Hoffman, Saenko, and Darrell]{tzeng2017adversarial}
Tzeng, E., Hoffman, J., Saenko, K., and Darrell, T.
\newblock Adversarial discriminative domain adaptation.
\newblock \emph{Proceedings of the IEEE conference on computer vision and pattern recognition}, 2017.

\bibitem[Veli{\v{c}}kovi{\'c} et~al.(2018)Veli{\v{c}}kovi{\'c}, Cucurull, Casanova, Romero, Li{\`o}, and Bengio]{velivckovic2018graph}
Veli{\v{c}}kovi{\'c}, P., Cucurull, G., Casanova, A., Romero, A., Li{\`o}, P., and Bengio, Y.
\newblock Graph attention networks.
\newblock \emph{International Conference on Learning Representations}, 2018.

\bibitem[Wang et~al.(2019)Wang, Lin, Cui, Jia, Wang, Fang, Yu, Zhou, Yang, and Qi]{wang2019semi}
Wang, D., Lin, J., Cui, P., Jia, Q., Wang, Z., Fang, Y., Yu, Q., Zhou, J., Yang, S., and Qi, Y.
\newblock A semi-supervised graph attentive network for financial fraud detection.
\newblock \emph{IEEE International Conference on Data Mining}, 2019.

\bibitem[Wang et~al.(2020)Wang, Shen, Huang, Wu, Dong, and Kanakia]{wang2020microsoft}
Wang, K., Shen, Z., Huang, C., Wu, C.-H., Dong, Y., and Kanakia, A.
\newblock Microsoft academic graph: When experts are not enough.
\newblock \emph{Quantitative Science Studies}, 2020.

\bibitem[Wang et~al.(2023)Wang, Wang, and Ying]{anonymous2023improved}
Wang, Q., Wang, Y., and Ying, X.
\newblock Improved invariant learning for node-level out-of-distribution generalization on graphs.
\newblock \emph{Submitted to The Twelfth International Conference on Learning Representations}, 2023.

\bibitem[Wang et~al.(2021)Wang, Wang, Liang, Cai, and Hooi]{wang2021mixup}
Wang, Y., Wang, W., Liang, Y., Cai, Y., and Hooi, B.
\newblock Mixup for node and graph classification.
\newblock \emph{Proceedings of the Web Conference}, 2021.

\bibitem[Wei et~al.(2022)Wei, Yin, Jia, Benson, and Li]{wei2022understanding}
Wei, R., Yin, H., Jia, J., Benson, A.~R., and Li, P.
\newblock Understanding non-linearity in graph neural networks from the bayesian-inference perspective.
\newblock \emph{Advances in Neural Information Processing Systems}, 2022.

\bibitem[Wu et~al.(2023)Wu, He, and Ainsworth]{wu2023non}
Wu, J., He, J., and Ainsworth, E.
\newblock Non-iid transfer learning on graphs.
\newblock \emph{Proceedings of the AAAI Conference on Artificial Intelligence}, 2023.

\bibitem[Wu et~al.(2020)Wu, Pan, Zhou, Chang, and Zhu]{wu2020unsupervised}
Wu, M., Pan, S., Zhou, C., Chang, X., and Zhu, X.
\newblock Unsupervised domain adaptive graph convolutional networks.
\newblock \emph{Proceedings of The Web Conference}, 2020.

\bibitem[Wu et~al.(2022)Wu, Zhang, Yan, and Wipf]{wu2022handling}
Wu, Q., Zhang, H., Yan, J., and Wipf, D.
\newblock Handling distribution shifts on graphs: An invariance perspective.
\newblock \emph{International Conference on Learning Representations}, 2022.

\bibitem[Wu et~al.(2021)Wu, Wang, Zhang, He, and Chua]{wu2021discovering}
Wu, Y., Wang, X., Zhang, A., He, X., and Chua, T.-S.
\newblock Discovering invariant rationales for graph neural networks.
\newblock \emph{International Conference on Learning Representations}, 2021.

\bibitem[Xu et~al.(2018)Xu, Hu, Leskovec, and Jegelka]{xu2018powerful}
Xu, K., Hu, W., Leskovec, J., and Jegelka, S.
\newblock How powerful are graph neural networks?
\newblock \emph{International Conference on Learning Representations}, 2018.

\bibitem[Yang et~al.(2022)Yang, Zeng, Wu, Jia, and Yan]{yang2022learning}
Yang, N., Zeng, K., Wu, Q., Jia, X., and Yan, J.
\newblock Learning substructure invariance for out-of-distribution molecular representations.
\newblock \emph{Advances in Neural Information Processing Systems}, 2022.

\bibitem[Yehudai et~al.(2021)Yehudai, Fetaya, Meirom, Chechik, and Maron]{yehudai2021local}
Yehudai, G., Fetaya, E., Meirom, E., Chechik, G., and Maron, H.
\newblock From local structures to size generalization in graph neural networks.
\newblock \emph{International Conference on Machine Learning}, 2021.

\bibitem[Yin et~al.(2022)Yin, Shen, Li, Wang, Luo, Chen, Luo, and Hua]{yin2022deal}
Yin, N., Shen, L., Li, B., Wang, M., Luo, X., Chen, C., Luo, Z., and Hua, X.-S.
\newblock Deal: An unsupervised domain adaptive framework for graph-level classification.
\newblock \emph{Proceedings of the 30th ACM International Conference on Multimedia}, 2022.

\bibitem[Yin et~al.(2023)Yin, Shen, Wang, Lan, Ma, Chen, Hua, and Luo]{yin2023coco}
Yin, N., Shen, L., Wang, M., Lan, L., Ma, Z., Chen, C., Hua, X.-S., and Luo, X.
\newblock Coco: A coupled contrastive framework for unsupervised domain adaptive graph classification.
\newblock \emph{Internationl Conference on Machine Learning}, 2023.

\bibitem[You et~al.(2023)You, Chen, Wang, and Shen]{you2023graph}
You, Y., Chen, T., Wang, Z., and Shen, Y.
\newblock Graph domain adaptation via theory-grounded spectral regularization.
\newblock \emph{International Conference on Learning Representations}, 2023.

\bibitem[Yu et~al.(2020)Yu, Xu, Rong, Bian, Huang, and He]{yu2020graph}
Yu, J., Xu, T., Rong, Y., Bian, Y., Huang, J., and He, R.
\newblock Graph information bottleneck for subgraph recognition.
\newblock \emph{International Conference on Learning Representations}, 2020.

\bibitem[Zellinger et~al.(2016)Zellinger, Grubinger, Lughofer, Natschl{\"a}ger, and Saminger-Platz]{zellinger2016central}
Zellinger, W., Grubinger, T., Lughofer, E., Natschl{\"a}ger, T., and Saminger-Platz, S.
\newblock Central moment discrepancy (cmd) for domain-invariant representation learning.
\newblock \emph{International Conference on Learning Representations}, 2016.

\bibitem[Zhang et~al.(2013)Zhang, Sch{\"o}lkopf, Muandet, and Wang]{zhang2013domain}
Zhang, K., Sch{\"o}lkopf, B., Muandet, K., and Wang, Z.
\newblock Domain adaptation under target and conditional shift.
\newblock \emph{International Conference on Machine Learning}, 2013.

\bibitem[Zhang et~al.(2021)Zhang, Du, Xie, and Wang]{zhang2021adversarial}
Zhang, X., Du, Y., Xie, R., and Wang, C.
\newblock Adversarial separation network for cross-network node classification.
\newblock \emph{Proceedings of the 30th ACM International Conference on Information \& Knowledge Management}, 2021.

\bibitem[Zhang et~al.(2019)Zhang, Song, Du, Yang, and Jin]{zhang2019dane}
Zhang, Y., Song, G., Du, L., Yang, S., and Jin, Y.
\newblock Dane: Domain adaptive network embedding.
\newblock \emph{IJCAI International Joint Conference on Artificial Intelligence}, 2019.

\bibitem[Zhao et~al.(2018)Zhao, Zhang, Wu, Moura, Costeira, and Gordon]{zhao2018adversarial}
Zhao, H., Zhang, S., Wu, G., Moura, J.~M., Costeira, J.~P., and Gordon, G.~J.
\newblock Adversarial multiple source domain adaptation.
\newblock \emph{Advances in Neural Information Processing Systems}, 2018.

\bibitem[Zhao et~al.(2019)Zhao, Des~Combes, Zhang, and Gordon]{zhao2019learning}
Zhao, H., Des~Combes, R.~T., Zhang, K., and Gordon, G.
\newblock On learning invariant representations for domain adaptation.
\newblock \emph{International Conference on Machine Learning}, 2019.

\bibitem[Zhu et~al.(2021)Zhu, Ponomareva, Han, and Perozzi]{zhu2021shift}
Zhu, Q., Ponomareva, N., Han, J., and Perozzi, B.
\newblock Shift-robust gnns: Overcoming the limitations of localized graph training data.
\newblock \emph{Advances in Neural Information Processing Systems}, 2021.

\bibitem[Zhu et~al.(2022)Zhu, Zhang, Park, Yang, and Han]{zhu2022shift}
Zhu, Q., Zhang, C., Park, C., Yang, C., and Han, J.
\newblock Shift-robust node classification via graph adversarial clustering.
\newblock \emph{arXiv preprint arXiv:2203.15802}, 2022.

\bibitem[Zhu et~al.(2023)Zhu, Jiao, Ponomareva, Han, and Perozzi]{zhu2023explaining}
Zhu, Q., Jiao, Y., Ponomareva, N., Han, J., and Perozzi, B.
\newblock Explaining and adapting graph conditional shift.
\newblock \emph{arXiv preprint arXiv:2306.03256}, 2023.

\end{thebibliography}
\bibliographystyle{icml2024}


\newpage
\appendix
\onecolumn
\section{Some Definitions}
\label{app:CSBM}
\begin{definition}[Contextual Stochastic Block Model]~\cite{deshpande2018contextual} 
    \label{def:CSBM}

    The Contextual Stochastic Block Model (CSBM) is a framework combining the stochastic block model with node features for random graph generation. A CSBM with nodes belonging to $k$ classes is defined by parameters $(n, \mB, \mathbb{P}_0, \dots, \mathbb{P}_{k-1})$, where $n$ represents the total number of nodes. The matrix $\mB$, a $k \times k$ matrix, denotes the edge connection probability between nodes of different classes. Each $\mathbb{P}i$ (for $0 \leq i < k$) characterizes the feature distribution of nodes from class $i$. In a graph generated from CSBM, the probability that an edge exists between a node $u$ from class $i$ and a node $v$ from class $j$ is specified by $B_{ij}$, an element of $\mB$. For undirected graphs, $\mB$ is symmetric, i.e., $\mB=\mB^\top$. In CSBM, node features and edges are generated independently, conditioned on node labels.
\end{definition}

\section{Omitted Proofs}

\subsection{Proof for Theorem \ref{thm:decompose}}
\mythm*

\begin{proof}
     We analyze the distribution $\bP_\dU(H^{(k+1)}|Y)$ to see which distributions should be aligned to achieve $\bP_\dS(H^{(k+1)}|Y) = \bP_\dT(H^{(k+1)}|Y)$. Since $h_u^{(k+1)} = \text{UPT}\,(h_u^{(k)}, \text{AGG}\,(\ldbb h_v^{(k)}:v \in \dN_u\rdbb))$, $\bP_\dU(H^{(k+1)}|Y)$ can be expanded as follows:
    \begin{align}
        \notag&\bP_\dU(h_u^{(k)}, \ldbb h_v^{(k)}:v \in \dN_u\rdbb | Y_u=i)\\
        \notag&\stackrel{(a)}{=} \bP_\dU(h_u^{(k)}|Y_u = i)\bP_\dU(\ldbb h_v^{(k)}:v \in \dN_u\rdbb | Y_u=i)\\
        \notag& \stackrel{}{=} \bP_\dU(h_u^{(k)}|Y_u = i)\bP_\dU(|\dN_u| = d | Y_u=i)\bP_\dU(\ldbb h_v^{(k)}\rdbb | Y_u=i, v \in \dN_u, |\dN_u| = d)\\
        \notag& \stackrel{}{=} \bP_\dU(h_u^{(k)}|Y_u = i)\bP_\dU(|\dN_u| = d | Y_u=i)\bP_\dU(\ldbb h_{v_1}^{(k)}, \cdots h_{v_d}^{(k)}\rdbb | Y_u=i, v_t \in \dN_u, \text{for} \;t\in [1,d])\\
        \notag& \stackrel{(b)}{=} \bP_\dU(h_u^{(k)}|Y_u = i)\bP_\dU(|\dN_u| = d | Y_u=i)(d\,!)\prod_{t = 1}^d\bP_\dU(h_{v_t}^{(k)}| h_{v_{1:t-1}}^{(k)}, Y_u=i, v_t \in \dN_u)\\
        \notag& \stackrel{}{=} \bP_\dU(h_u^{(k)}|Y_u = i)\bP_\dU(|\dN_u| = d | Y_u=i)(d\,!)\prod_{t = 1}^d(\sum_{j\in \dY}\bP_\dU(h_{v_t}^{(k)}| Y_{v_t} = j, h_{v_{1:t-1}}^{(k)}, Y_u=i, v_t \in \dN_u)\\
        \notag &\quad \bP_\dU(Y_{v_t} = j|h_{v_{1:t-1}}^{(k)}, Y_u=i, v_t \in \dN_u))\\
         \label{eq:decompose} & \stackrel{(c)}{=} \bP_\dU(h_u^{(k)}|Y_u = i)\bP_\dU(|\dN_u| = d | Y_u=i)(d\,!)\prod_{t = 1}^d(\sum_{j\in \dY}\bP_\dU(h_{v_t}^{(k)}| Y_{v_t} = j)\bP_\dU(Y_{v_t} = j|Y_u=i, v_t \in \dN_u))
    \end{align}

    
    (a) is based on the assumption that node attributes and edges are conditionally independent of others given the node labels.  
    (b), here we suppose that the observed messages $h_{v}$ are different $\forall v \in \dN_u $, and this assumption does not affect the result of the theorem. If some of them are identical, we modify the coefficient $d!$ as $\frac{d!}{\Pi_{t=1}^dm_i!}$, where $m_t$ denotes the repeated messages. For simplicity, we assume that $m_t = 1, \forall t\in [1, d]$. (c) is based on the assumption that given $Y=y_u$, $h_u^{(k)}$ is independently sampled from $\bP_\dU(H^{(k)}|Y)$

    With the goal to achieve $\bP_\dS(H^{(k+1)}|Y) = \bP_\dT(H^{(k+1)}|Y)$, it suffices to achieve by making the input distribution equal across the source and the target $$\bP_\dS(h_u^{(k)}, \ldbb h_u^{(k)}:v \in \dN_u\rdbb | Y_u=i) = \bP_\dT(h_u^{(k)}, \ldbb h_u^{(k)}:v \in \dN_u\rdbb | Y_u=i)$$
    since the source and target graphs undergo the same set of functions. Based on Eq.~\eqref{eq:decompose},this means it suffices to let  $\bP_\dS(h_u^{(k)}|Y_u = i) = \bP_\dT(h_u^{(k)}|Y_u = i)$ and $\bP_\dS(h_{v_t}^{(k)}| Y_{v_t} = j) = \bP_\dT(h_{v_t}^{(k)}| Y_{v_t} = j)$ since $\bP_\dS(H^{(k)}|Y) = \bP_\dT(H^{(k)}|Y)$ is assumed to be true. Therefore, as long as there exists a transformation that modifies the $\dN_u \rightarrow \Tilde{\dN}_u$ such that $$\bP_\dS(|\Tilde{\dN}_u| = d | Y_u=i) = \bP_\dT(|\dN_u| = d | Y_u=i) ;\quad \bP_\dS(Y_v = j|Y_u=i, v \in \Tilde{\dN}_u)) = \bP_\dT(Y_v = j|Y_u=i, v \in \dN_u))$$
    Then, $\bP_\dS(H^{(k+1)}|Y) = \bP_\dT(H^{(k+1)}|Y)$
\end{proof}

\begin{remark}
   Iteratively, we can achieve $\bP_\dS(H^{(L)}|Y) = \bP_\dT(H^{(L)}|Y)$ given no feature shift initially $\bP_\dS(X|Y) = \bP_\dT(X|Y)$ as $\bP_\dS(H^{(1)}| Y) = \bP_\dT(H^{(1)}| Y)$
    \begin{gather*}
        \text{base case: } \bP_\dS(H^{(1)}| Y) = \bP_\dT(H^{(1)}| Y) \Rightarrow \bP_\dS(H^{(2)}| Y) = \bP_\dT(H^{(2)}| Y)\\
        \text{inductive step: } \bP_\dS(H^{(k)}| Y) = \bP_\dT(H^{(k)}| Y),  \stackrel{(d)}{\Rightarrow} \bP_\dS(H^{(k+1)}| Y) = \bP_\dT(H^{(k+1)}| Y)\\
        \text{Therefore, }\bP_\dS(H^{(L)}| Y) = \bP_\dT(H^{(L)}| Y).
    \end{gather*}
    (d) is proved above that when using a multiset transformation to align two distributions, this can be guaranteed 

    Under the assumption that given $Y=y_u$, $h_u^{(k)}$ is independently sampled from $\bP_\dU(H^{(k)}|Y)$, $\bP_\dS(H^{(L)}|Y) = \bP_\dT(H^{(L)}|Y)$ can induce $\bP_\dS(\mH^{(L)}|\mY) = \bP_\dT(\mH^{(L)}|\mY)$ since $\bP(\mH^{(L)} = \mh^{(L)}|\mY = \my) = \Pi_{u\in \dV}\bP(H^{(L)} = h_u|Y = y_u)$
\end{remark}

\subsection{Proof for Lemma \ref{lem:optw}}
\wlemma*
\begin{proof}
    \begin{align*}
        \bP_\dT(\hY_u = i, \hY_v = j|A_{uv} = 1) &= \sum_{i',j' \in \dY}\bP_\dT(\hY_u = i, \hY_v = j|Y_u = i', Y_v = j', A_{uv} = 1)\bP_\dT(Y_u = i', Y_v = j'|A_{uv} = 1)\\
        &\stackrel{(a)}{=} \sum_{i',j' \in \dY}\bP_\dT(\hY_u = i|Y_u = i')\bP_\dT(\hY_v = j|Y_v = j')\bP_\dT(Y_u = i', Y_v = j'|A_{uv} = 1)\\
        &\stackrel{(b)}{=} \sum_{i',j' \in \dY}\bP_\dS(\hY_u = i|Y_u = i')\bP_\dS(\hY_v = j|Y_v = j')\bP_\dT(Y_u = i', Y_v = j'|A_{uv} = 1)\\
        &= \sum_{i',j' \in \dY}\bP_\dS(\hY_u = i, \hY_v = j|Y_u = i', Y_v = j', A_{uv} = 1)\bP_\dT(Y_u = i', Y_v = j'|A_{uv} = 1)\\
        &= \sum_{i',j' \in \dY}\bP_\dS(\hY_u = i, \hY_v = j, Y_u = i', Y_v = j'| A_{uv} = 1)\frac{\bP_\dT(Y_u = i', Y_v = j'|A_{uv} = 1)}{\bP_\dS(Y_u = i', Y_v = j'|A_{uv} = 1)}\\
        &= \sum_{i', j' \in \dY}[\mSigma]_{ij, i'j'}[\mw]_{i'j'}
    \end{align*}
    (a) is because $\hy_u = g(h_u^{(L)})$ and the assumption that node representations and graph structures are conditionally independent of others given the node labels. And (b) is achieved since $\bP_\dS(H^{(L)}|Y) = \bP_\dT(H^{(L)}|Y)$ is satisfied, such that $\bP_\dS(g(h_u^{(L)}) = i|Y_u = i') = \bP_\dT(g(h_u^{(L)}) = i|Y_u = i'), \forall i' \in \dY$
\end{proof}

\subsection{Proof for Lemma \ref{lem:optbeta}}
\betalemma*

\begin{proof}
    \begin{align*}
        \bP_\dT(\hY_u = i) &= \sum_{i'\in \dY}\bP_\dT(\hY_u = i|Y_u = i')\bP_\dT(Y_u = i')\\
        &\stackrel{(a)}{=} \sum_{i'\in \dY}\bP_\dS(\hY_u = i|Y_u = i')\bP_\dT(Y_u = i')\\
        &= \sum_{i'\in \dY}\bP_\dS(\hY_u = i, Y_u = i')\frac{\bP_\dT(Y_u = i')}{\bP_\dS(Y_u = i')}\\
        &= \sum_{i'\in \dY}[\mC]_{i, i'}[\mbeta]_{i'}
    \end{align*}
    (a) is because, when $\bP_\dS(H^{(k+1)}|Y) = \bP_\dT(H^{(k+1)}|Y)$ is satisfied, $\bP_\dS(g(h_u^{(L)}) = i|Y_u = i') = \bP_\dT(g(h_u^{(L)}) = i|Y_u = i'), \forall i' \in \dY$
\end{proof}

\section{Algorithm Details}

\subsection{Details in optimization for $\mgamma$}
\label{app:calcw}
\subsubsection{Empirical estimation of $\mSigma$ and $\mnu$ in matrix form}

For the least square problem that solves for $\mw$
\begin{align*}
    \mSigma \mw = \mnu
\end{align*}
where $\mSigma \in \mathbb{R}^{|\dY|^2 \times |\dY|^2}$, $\mw \in \mathbb{R} ^ {|\dY|^2 \times 1}$, $\mnu \in \mathbb{R} ^ {|\dY|^2 \times 1}$ 

Empirically, we estimate the value of $\hat{\mSigma}$ and $\hat{\mnu}$ as following:
\begin{align*}
    \hat{\mSigma} = \frac{1}{|\dE_\dS|}\mE^\dS \mM^\dS
\end{align*}
$\mE^\dS \in \mathbb{R}^{|\dY|^2 \times |\dE_\dS|}$, where each column represents the joint distribution of the classes prediction associated with the starting and ending node of each edge in the source graph. $[\mE^\dS]_{:, {uv}} = g(h_u^{(L)}) \otimes g(h_v^{(L)}), \forall \text{edge}\;{uv} \in \dE_\dS$. And each entry $[\mE^\dS]_{{ij}, {uv}} = [g(h_u^{(L)})]_i \times [g(h_v^{(L)})]_j, \forall i,j \in \dY$.
$\mM^\dS \in \mathbb{R}^{|\dE_\dS| \times |\dY|^2}$ encodes the ground truth of the starting and ending node of an edge, as $[\mM^\dS]_{uv, y_uy_v} = 1$ for each edge $uv \in \dE_\dS$.

\begin{align*}
    \hat{\mnu} = \frac{1}{|\dE_\dT|}\mE^\dT \mone
\end{align*}
Similarly, $\mE^\dT \in \mathbb{R}^{|\dY|^2 \times |\dE_\dT|}$, where each column represents the joint distribution of the classes prediction associated with the starting and ending node of each edge in the target graph. $[\mE^\dT]_{:, {uv}} = g(h_u^{(L)}) \otimes g(h_v^{(L)}), \forall \text{edge}\;{uv} \in \dE_\dT$. And each entry $[\mE^\dT]_{{ij}, {uv}} = [g(h_u^{(L)})]_i \times [g(h_v^{(L)})]_j, \forall i,j \in \dY$. $\mone \in \mathbb{R}^{|\dE_\dT| \times 1}$ is the all one vector.

\subsubsection{Calculate for $\malpha$ in matrix form}

To finally solve for the ratio weight $\mgamma$, we need the value $\malpha$. 
\begin{align*}
    \malpha_i = \frac{\bP_\dT(y_u=i|A_{uv}=1)}{\bP_\dS(y_u=i|A_{uv}=1)} &= \frac{\sum_j\bP_\dT(y_u=i, y_v=j|A_{uv}=1)}{\sum_j\bP_\dS(y_u=i, y_v=j|A_{uv}=1)}\\
    &= \frac{\sum_j\frac{\bP_\dT(y_u = i, y_v = j|A_{uv} = 1)}{\bP_\dS(y_u = i, y_v = j|A_{uv} = 1)}\bP_\dS(y_u=i, y_v=j|A_{uv}=1)}{\sum_j\bP_\dS(y_u=i, y_v=j|A_{uv}=1)}\\
    &= \frac{\sum_j\frac{\bP_\dT(y_u = i, y_v = j|A_{uv} = 1)}{\bP_\dS(y_u = i, y_v = j|A_{uv} = 1)}\bP_\dS(y_u=i, y_v=j|A_{uv}=1)}{\bP_\dS(y_u=i|A_{uv}=1)}
\end{align*}
In matrix form, we construct $\mK \in \mathbb{R}^{|\dY| \times |\dY|^2}$, where $[K]_{i, {ij}} = \frac{\bP_\dS(y_u=i, y_v=j|A_{uv}=1)}{\bP_\dS(y_u=i|A_{uv}=1)}, \forall i, j \in |\dY|$. Note that $[K]_{i, {i'j}} = 0$ for $i' \neq i, \forall j \in |\dY|$. 
\begin{align*}
    \malpha = \mK\mw
\end{align*}

\subsection{Details in optimization for $\mbeta$}
\label{app:calcbeta}
For the least square problem that solves for $\mbeta$
\begin{align*}
    \mC \mbeta = \mmu
\end{align*}
where $\mC \in \mathbb{R}^{|\dY| \times |\dY|}$, $\mbeta \in \mathbb{R} ^ {|\dY| \times 1}$, $\mmu \in \mathbb{R} ^ {|\dY| \times 1}$ 

Empirically, we estimate the value of $\hat{\mC}$ and $\hat{\mmu}$ in matrix form as following:
\begin{align*}
    \hat{\mC} = \frac{1}{|\dV_\dS|}\mD^\dS \mL^\dS
\end{align*}
$\mD^\dS \in \mathbb{R}^{|\dY| \times |\dV_\dS|}$, where each column represents the distribution of the class prediction of each node in the source graph. $[\mD^\dS]_{:, u} = g(h_u^{(L)}), \forall u \in \dV_\dS$. And each entry $[\mD^\dS]_{{i}, {u}} = [g(h_u^{(L)})]_i, \forall i \in \dY$.
$\mL^\dS \in \mathbb{R}^{|\dV_\dS| \times |\dY|}$ that encodes the ground truth class of each node, as $[\mL^\dS]_{u, y_u} = 1$ for each node $u \in \dV_\dS$.

\begin{align*}
    \hat{\mmu} = \frac{1}{|\dV_\dT|}\mD^\dT \mone
\end{align*}
Similarly, $\mD^\dT \in \mathbb{R}^{|\dY| \times |\dV_\dT|}$, where each column represents the distribution of the class prediction of each node in the target graph. $[\mD^\dT]_{:, u} = g(h_u^{(L)}), \forall u \in \dV_\dT$. And each entry $[\mD^\dT]_{{i}, {u}} = [g(h_u^{(L)})]_i, \forall i \in \dY$. $\mone \in \mathbb{R}^{|\dV_\dT| \times 1}$ is the all one vector.

\section{More Related Works}
\label{app:morerelated}
\textbf{Other node-level DA works}  Other domain invariant learning-based methods, like \citet{shen2020network} proposed to align the class-conditioned representations with conditional MMD distance by using pseudo-label predictions for the target domain,   \citet{zhang2021adversarial} aimed to use separate networks to capture the domain-specific features in addition to a shared encoder for adversarial training and further \citet{pang2023sa} transformed the node features into spectral domain through Fourier transform for alignment. Other approaches like \citet{cai2021graph} disentangled semantic, domain, and noise variables and used semantic variables that are better aligned with target graphs for prediction. \citet{liu2024rethinking} explored the role of GNN propagation layers and linear transformation layers, thus proposing to use a shared transformation layer with more propagation layers on the target graph instead of a shared encoder.

\textbf{Node-level OOD works} In addition to GDA, many works target the out-of-distribution (OOD) generalization without access to unlabeled target data. For the node classification task, EERM~\cite{wu2022handling} and  LoRe-CIA~\cite{anonymous2023improved} both extended the idea of invariant learning to node-level tasks, where EERM minimized the variance over representations across different environments and LoRe-CIA enforced the cross-environment Intra-class Alignment of node representations to remove their reliance on spurious features. \citet{wang2021mixup} extended mixup to the node representation under node and graph classification tasks.

\textbf{Graph-level DA and OOD works} The shifts and methods in graph-level problems are significantly different from those for node-level tasks. The shifts in graph-level tasks can be modeled as IID by considering individual graphs and often satisfy the covariate shift assumption, which makes some previous IID works applicable. Under the availability of target graphs, there are several graph-level GDA works like~\cite{yin2023coco, yin2022deal}, where the former utilized contrastive learning to align the graph representations with similar semantics and the latter employed graph augmentation to match the target graphs under adversarial training. Regarding the scenarios in which we do not have access to the target graphs, it becomes the graph OOD problem. A dominant line of work in graph-level OOD is based on invariant learning originating from causality to identify a subgraph that remains invariant across graphs under distribution shifts. Among these works, \citet{wu2021discovering, chen2022learning, li2022learning, yang2022learning, chen2023does, gui2023joint, fan2022debiasing, fan2023generalizing} aimed to find the invariant subgraph, and \citet{miao2022interpretable, yu2020graph} used graph information bottleneck. Furthermore, another line of works adopted graph augmentation strategies, like~\cite{sui2023unleashing, jin2022empowering} and some mixup-based methods~\cite{han2022g, ling2023graph, jia2023graph}. Moreover, some works focused on handling the size shift~\cite{yehudai2021local, bevilacqua2021size, chuang2022tree}.

\section{Experiments details}
\label{app:exp}
\subsection{Dataset Details}
\label{app:dataset}

\textbf{Dataset Statistics} Here we report the number of nodes, number of edges, feature dimension, and the number of labels for each dataset. The Arxiv-year means the graph with papers till that year. The edges are all undirected edges, which are counted twice in the edge list. 

\begin{table}[h!]
\caption{real dataset statistics}
\vspace{2mm}
\begin{center}
\begin{sc}
\begin{tabular}{lcccccc}
\toprule
            & ACM          & DBLP           & Arxiv-2007          &Arxiv-2009  & Arxiv-2016          &Arxiv-2018\\
\midrule
$\#$nodes          & $7410$  & $5578$ & $4980$ &  $9410$ & $69499$ & $120740$\\
$\#$edges        & $11135    $ & $7341$ & $5849$    &$13179$       & $232419$ & $615415$\\
Node feature dimension    & $7537    $ & $7537  $ & $128$   &    $128$  & $128$   &    $128$   \\
$\#$labels        & $6    $ & $6  $ & $40$    &$40$  & $40$    &$40$      \\
\bottomrule
\label{table:datastats}
\end{tabular}
\end{sc}
\end{center}
\vskip -0.7cm
\end{table}

\begin{table}[h!]
\caption{MAG dataset statistics}
\vspace{2mm}
\begin{center}
\begin{sc}
\begin{tabular}{lcccccc}
\toprule
            & US          & CN           & DE          &JP  & RU          &FR\\
\midrule
$\#$nodes          & $132558$  & $101952$ & $43032$     &  $37498$  & $32833$     &  $29262$   \\
$\#$edges        & $697450$ & $285561$ & $126683$    &$90944$    & $67994$     &  $78222$    \\
Node feature dimension    & $128    $ & $128  $ & $128$   &    $128$  & $128$   &    $128$   \\
$\#$labels        & $20$ & $20  $ & $20$    &$20$  & $20$    &$20$ \\
\bottomrule
\label{table:datastats}
\end{tabular}
\end{sc}
\end{center}
\vskip -0.7cm
\end{table}

\begin{table}[h!]
\caption{Pileup dataset statistics}
\vspace{2mm}
\begin{center}
\begin{sc}
\begin{tabular}{lcccccc}
\toprule
            & gg-10          & qq-10           & gg-30          &qq-30  &gg-50  & gg-140          \\
\midrule
$\#$nodes          & $18611$  & $17242$ & $41390$     &  $38929$ & $60054$     &  $154750$     \\
$\#$edges        & $ 53725 $ & $42769$ & $173392$    &$150026$  & $341930$    &$2081229$    \\
Node feature dimension    & $28    $ & $28  $ & $28$   &    $28$  & $28$   &    $28$   \\
$\#$labels        & $2    $ & $2  $ & $2$    &$2$  & $2$    &$2$      \\
\bottomrule
\label{table:datastats}
\end{tabular}
\end{sc}
\end{center}
\vskip -0.7cm
\end{table}

\textbf{DBLP and ACM} are two paper citation networks obtained from DBLP and ACM, originally from~\cite{tang2008arnetminer} and processed by~\cite{wu2020unsupervised}. We use the processed version. Nodes are papers and undirected edges represent citations between papers. The goal is to predict the 6 research topics of a paper: “Database”, “Data mining”, “Artificial intelligent”, “Computer vision”, “Information Security” and
"High Performance Computing".

\textbf{Arxiv} introduced in~\cite{hu2020open} is another citation network of Computer Science (CS) Arxiv papers to predict 40 classes on different subject areas. The feature vector is a 128-dimensional word2vec vector with the average embedding of the paper's title and abstract. Originally it is a directed graph with directed citations between papers, we convert it into an undirected graph. 

\subsubsection{More details MAG datasets}
MAG is a subset of the Microsoft Academic Graph (MAG) as detailed in~\cite{hu2020open, wang2020microsoft}, originally containing entities as papers, authors, institutions, and fields of study. There are four types of directed relations in the original graph connecting two types of entities: an author "is affiliated with" an institution, an author "writes" a paper, a paper "cites" a paper, and a paper "has a topic of" a field of study. The node feature for a paper is the word2vec vector with 128 dimensions. The task is to predict the publication venue of papers, which in total has 349 classes. We curate the graph to include only paper nodes and convert directed citation links to undirected edges. Papers are split into separate graphs based on the country of the institution the corresponding author is affiliated with. Then, we detail the process of generating a separate ``paper-cites-paper'' homogeneous graph for each country from the original ogbn-mag dataset.

\textbf{Determine the country of origin for each paper.} The rule of determining the country of the paper is based on the country of the institute the corresponding author is affiliated with. Since the original ogbn-mag dataset does not indicate the information of the corresponding author, we retrieve the metadata of the papers via \href{https://openalex.org/}{OpenAlex},\footnote{This is an alternative way considering the \href{https://www.microsoft.com/en-us/research/project/academic/articles/microsoft-academic-to-expand-horizons-with-community-driven-approach/}{Microsoft Academic website and underlying APIs have been retired on Dec. 31, 2021.}}. Specifically, there is a boolean variable on OpenAlex boolean indicating whether an author is the corresponding author for each paper. Then, we further locate the institution this corresponding author is affiliated with and retrieve that institution's country to use as the country code for the paper. All these operations can be done through OpenAlex. However, not all papers include this corresponding author information on OpenAlex. Regarding the papers that miss this information, we determine the country of this paper through a majority vote based on the institution country of all authors in this paper. Namely, we first identify all authors recorded in the original dataset via the ``author---writes---paper'' relation and acquire the institute information for these authors through the relation of ``author---is\_affiliated\_with---institution''. Then, with the country information retrieved from OpenAlex for these institutions, we do a majority vote to determine the final country code for the paper. 



\textbf{Generate country-specific graphs.} Based on the country information obtained above, we generate a separate citation graph for a given country $C$. It will contain all papers that have a country code of $C$ and the edges indicating the citation relationships within these papers. The edge\_index set $\cal E$ is initialized as $\varnothing$. For each citation pair $(v_i,v_j)$ in the original ``paper-cites-paper'' graph, it is added to $\cal E$ \textit{iff.} both $v_i$ and $v_j$ have the same country affiliation $C$. We then obtain the node set $\cal V$ based on all unique nodes appearing in $\cal E$. In the scope of this work, we only focus on the top 19 publication venues with the most papers for classification and combine the rest of the classes into a single dummy class.

\subsubsection{More details for HEP datasets}
\label{app:morehepdata}
Initially, there are multiple graphs with each graph representing a collision event in the large hadron collider (LHC). Here, we collate the graphs together to form a single large graph. We use 100 graphs in each domain to create the single source and target graph respectively. In the source graph, the nodes in 60 graphs are used for training, 20 are used for validation and 20 are used for testing. In the target graph, the nodes in 20 graphs are used for validation and 80 are used for testing. The particles can be divided into charged and neutral particles, where the labels of the charged particles are known by the detector. Therefore, the classifications are only done on the neutral particles. The node features contain the particle's position in $\eta$ axis, pt as energy, the pdgID one hot encoding to indicate the type of particle, and the label of the particle (label for changed, unknown for neutral) to help with classification as neighborhood information. 

Pileup (PU) levels indicate the number of other collisions in the background event, it is closely related to the label distribution of LC and OC. For instance, a high PU graph will have mostly OC particles and few LC particles. Also, it will cause significant CSS as the distribution of particles easily influences the connections between them. The physical processes correspond to different types of signal decay of the particles, which mainly causes some slight feature shifts and nearly no LS or CSS under the same PU level. 

\subsection{Detailed experimental setting}
\label{app:expset}

\textbf{Model architecture}
The backbone model is GraphSAGE with mean pooling having 3 GNN layers and 2 MLP layers for classification. The hidden dimension for GNN is 300 for Arxiv and MAG, 50 for Pileup,  128 for the DBLP/ACM dataset and 20 for synthetic datasets. The classifier dimension 300 for Arxiv and MAG, 50 for Pileup,  40 for DBLP/ACM dataset and 20 for synthetic datasets. If there is adversarial training with a domain classifier for some baselines, it has 3 layers and the hidden dimension is the same as the GNN dimension. All experiments are repeated three times. 

\textbf{Hardware} All experiments are run on NVIDIA RTX A6000 with 48G memory and Quadro RTX 6000 with 24G memory. Specifically, for the UDAGCN baselines, we try with the 48G memory GPU but still out of memory. 

\textbf{Synthetic Datasets}
The synthetic dataset is generated under the contextual stochastic block model (CSBM), where there are in total of 6000 nodes and 3 classes. We vary the edge connection probability matrix and the node label distribution in different settings. The node features are generated from a Gaussian distribution where $\bP_0 = \dN([1, 0, 0], \sigma^2I)$, $\bP_1 = \dN([0, 1, 0], \sigma^2I)$ and $\bP_2 = \dN([0, 0, 1], \sigma^2I), \sigma = 0.3$, and the distribution is the same for the source and target graph in all settings. 
We denote the format of edge connection probability matrix as $ \mB=\left[\begin{array}{ccc}
    p & q & q\\ 
    q & p & q\\
    q & q & p\end{array}\right]$, where $p$ is the intra-class edge probability and $q$ is the inter-class edge probability. 
\begin{itemize}

    \item The source graph has $\bP_Y = [1/3, 1/3, 1/3]$ and $p = 0.02, q = 0.005$. 

    \item For setting 1 and 2 with the shift in only class ratio, they have the same $\bP_Y$, and setting 1 has $p = 0.015, q = 0.0075$ and setting 2 has $p = 0.01, q = 0.01$. 

    \item For setting 3 and 4 with the shift in only cardinality, they have the same $\bP_Y$, and setting 3 has $p = 0.02/2, q = 0.005/2$ and setting 4 has $p = 0.02/4, q = 0.005/4$.

    \item For setting 5 and 6 with the shift in both class ratio and cardinality, they have the same $\bP_Y$, and setting 5 has $p = 0.015/2, q = 0.0075/2$ and setting 6 has $p = 0.01/2, q = 0.01/2$.

    \item For setting 7 and 8 with shifts in both CSS and label shift, they have the same edge connection probability as $p = 0.015/2, q = 0.0075/2$ but different label distributions. Setting 7 has $\bP_Y = [0.5, 0.25, 0.25]$ and setting 8 has $\bP_Y = [0.1, 0.3, 0.6]$.

\end{itemize}

\textbf{Pileup}
Regarding the experiments studying the shift in pileup levels, the pair with PU10 and PU30 is from signal qq. The other two pairs with PU10 and PU50, PU30 and PU140 are from signal gg. The experiments that study the shift in physical processes are from the same PU level 10. Compared to the Pileup datasets used in the StruRW paper~\cite{liu2023structural}, we investigate the physical process shift with datasets from signal qq and signal gg instead of signal gg and signal $Z(\nu\nu)$. Also, we conduct more experiments to study the pileup shifts under the same physical process being signal qq (PU10 vs. PU30) or signal gg (PU10 vs. PU50 and PU30 vs. PU140). In addition, the StruRW paper treats each event as a single graph. They train the algorithm using multiple training graphs and adopt the edge weights as the average from each graph. In this paper, we collate the graphs for all events together for training and weight estimations.  

\textbf{Arxiv}
The graph is formed based on the ending year, meaning that the graph contains all nodes till the specified ending year. For instance, for the experiments where the source papers ended in 2007, the source graph contains all nodes and edges associated with papers that were published no later than 2007. Then, if the target years are from 2014 to 2016, then the entire target graph contains all papers published till 2016, but we only evaluate on the papers published from 2014 to 2016. 

\textbf{DBLP/ACM}
Since we observe that this dataset presents additional feature shift, so we additionally add adversarial layers to align the node representations. Basically, it is the combination of \proj with label-weighted adversarial feature alignment, and the hyperparameters with additional adversarial layers are the same with DANN and will be detailed below. Also, note that to systematically control the label shift degree in this relatively small graph ( $<$ 10000 nodes), the split of nodes for training/validation/testing is done regarding each class of nodes. This is slightly different from the data in previous papers using this dataset, so the results may not be directly comparable.

\subsection{Hyperparameter tuning}
\label{app:hyperparameter}

Hyperparameter tuning involves adjusting $\delta$ for edge probability regularization in $\mgamma$ calculation and $\lambda$ for L2 regularization in the least square optimizations for $\mw$ and $\mbeta$. Selecting $\delta$ correlates to the degree of structure shift and $\lambda$ is chosen based on the number of labels and classification performance. In datasets like Arxiv and MAG, where classification is challenging and labels are numerous, leading to ill-conditioned or rank-deficient confusion matrices, a larger $\lambda$ is required. For simpler tasks with fewer classes, like synthetic and low PU datasets, a lower $\lambda$ suffices. $\delta$ should be small for larger CSS (MAG and Pileup) and large with smaller CSS (Arxiv and physical process shift in Pileup) to counteract the spurious $\mgamma$ value that may caused by variance in edge formation. Below is the detailed range of hyperparameters.

The learning rate is 0.003 and the number of epochs is 400 for all experiments. The hyperparameters are tuned mainly for the robustness control, as the $\delta$ in regularizing edges and $\lambda$ in L2 regularization for optimization of $\mw$ and $\mbeta$. 

Here, for all datasets, $\lambda_{\mbeta}$ for $\mbeta$ is chosen from $\{0.005, 0.01, 0.1, 1, 5\}$ to reweight the ERM loss to handle the LS. Additionally, we also consider reweighting the ERM loss by source label distribution together. Specifically, we found it useful in the case with imbalanced training label distribution, like both directions in DBLP/ACM datasets, transitioning from high PU to low PU, and the Arxiv training with papers pre-2007 and pre-2009. In other cases, we do not reweight the ERM loss by source label distribution.
\begin{itemize}

\item For the synthetic datasets, the $\delta$ is selected from $\{1\mathrm{e}{-6}, 1\mathrm{e}{-5}, 1\mathrm{e}{-4}\}$, $\lambda_\mw$ is selected from $\{0.005, 0.01, 0.1\}$

\item For the MAG dataset, the $\delta$ is selected from $\{1\mathrm{e}{-5}, 1\mathrm{e}{-4}, 1\mathrm{e}{-3}\}$, $\lambda_\mw$ is selected from $\{0.1, 1, 5, 10\}$

\item For the DBLP/ACM dataset, the $\delta$ is selected from $\{5\mathrm{e}{-5}, 1\mathrm{e}{-4}, 5\mathrm{e}{-4}\}$, $\lambda_\mw$ is selected from $\{20, 25, 30\}$

\item For the Pileup dataset, regarding the settings with pileup shift, $\delta$ is selected from $\{1\mathrm{e}{-6}, 1\mathrm{e}{-5}, 1\mathrm{e}{-4}\}$, $\lambda_\mw$ is selected from $\{0.005, 0.01, 0.1, 1\}$. Regarding the settings with physical process shift, $\delta$ is selected from $\{1\mathrm{e}{-5}, 1\mathrm{e}{-4}, 5\mathrm{e}{-4}\}$, $\lambda_\mw$ is selected from $\{1, 5, 10, 20\}$

\item For the Arxiv dataset, regarding the settings with training data till 2007, the $\delta$ is selected from $\{5\mathrm{e}{-3}, 1\mathrm{e}{-2}, 3\mathrm{e}{-2}\}$, $\lambda_\mw$ is selected from $\{1, 2, 5\}$. Regarding the settings with training data till 2009, the $\delta$ is selected from $\{3\mathrm{e}{-2}, 5\mathrm{e}{-2}, 8\mathrm{e}{-2}\}$, $\lambda_\mw$ is selected from $\{15, 20, 25\}$. Regarding the settings with training data till 2011, the $\delta$ is selected from $\{3\mathrm{e}{-4}, 5\mathrm{e}{-4}, 8\mathrm{e}{-4}\}$, $\lambda_\mw$ is selected from $\{30, 50, 80\}$
\end{itemize}

\subsection{Baseline Tuning}
\label{app:baseline}

\begin{itemize}

\item For DANN, we tune two hyperparameters as the coefficient before the domain alignment loss and the max value of the rate added during the gradient reversal layer. The rate is calculated as $q = \min((\text{epoch}+1)/\text{nepochs}), \text{max-rate})$. For all datasets, DA loss coefficient is selected from $\{0.2, 0.5, 1\}$ and max-rate is selected from $\{0.05, 0.2, 1\}$.

\item For IWDAN, we tune three hyperparameters, the same two parameters as the coefficient before the domain alignment loss and the max value of the rate added during the gradient reversal layer. For all datasets, DA loss coefficient is selected from $\{0.5, 1\}$ and max-rate is selected from $\{0.05, 0.2, 1\}$. Also, we tune the coefficient to update the label weight calculated after each epoch as $(1-\lambda) * \text{new weight} + \lambda * \text{previous weight}$, where $\lambda$ is selected from $\{0, 0.5\}$.

\item For SpecReg, we totally tune for 5 hyperparameters and we follow the original hyperparameters for the dataset Arxiv and DBLP/ACM. For DBLP/ACM dataset, $\gamma_{\text{adv}}$ is selected from $\{0.01, 0.2\}$, $\gamma_{\text{smooth}}$ is selected from $\{0.01, 0.1\}$, threshold-smooth is selected from $\{0.01, -1\}$, $\gamma_{\text{mfr}}$ is selected from $\{0.01, 0.1\}$, threshold-mfr is selected from $\{0.75, -1\}$. For Arxiv dataset, $\gamma_{\text{adv}}$ is selected from $\{0.01\}$, $\gamma_{\text{smooth}}$ is selected from $\{0, 0.1\}$, threshold-smooth is selected from $\{0, 1\}$, $\gamma_{\text{mfr}}$ is selected from $\{0, 0.1\}$, threshold-mfr is selected from $\{0, 1\}$. For the other datasets, $\gamma_{\text{adv}}$ is selected from $\{0.01\}$, $\gamma_{\text{smooth}}$ is selected from $\{0.01, 0.1\}$, threshold-smooth is selected from $\{0.1, 1\}$, $\gamma_{\text{mfr}}$ is selected from $\{0.01, 0.1\}$, threshold-mfr is selected from $\{0.1, 1\}$. Note that for the DBLP and ACM datasets, we implement their module (following their published code) on top of GNN instead of the UDAGCN model for fair comparison among baselines. 

\item For UDAGCN, we also tune the two hyperparameters from DANN as the coefficient before the domain alignment loss and the max value of the rate added during the gradient reversal layer. The rate is calculated as $q = \min((\text{epoch}+1)/\text{nepochs}), \text{max-rate})$. For all datasets, DA loss coefficient is selected from $\{0.2, 0.5, 1\}$ and max-rate is selected from $\{0.05, 0.2, 1\}$.

\item For StruRW, we use the StruRW-ERM baseline and we tune the $\lambda$ that controls the edge weights in GNN as $(1-\lambda) + \lambda * \text{edge weight}$ with range $\{0.1, 0.3, 0.7, 1\}$ and the epochs to start reweighting the edges from $\{100, 200, 300\}$.

\end{itemize}

\subsection{Shift statistics of datasets}
\label{app:shift_stats}
We design two metrics to measure the degree of structure shift in terms of CSS and LS. 

The metric of CSS is based on the node label distribution in the neighborhood of each class of nodes as $\bP_\dU(Y_v| Y_u, v\in \dN_u)$. Specifically, we calculate the total variation distance of this conditional neighborhood node label distribution of each class $\forall i \in \dY$ as:
\begin{align*}
    &TV(\bP_\dS(Y_v| Y_u = i, v\in \dN_u), \bP_\dT(Y_v| Y_u = i, v\in \dN_u)) \\
    &= \frac{1}{2}\lVert\bP_\dS(Y_v| Y_u = i, v\in \dN_u) - \bP_\dT(Y_v| Y_u = i, v\in \dN_u)\rVert_1\\
    &=\frac{1}{2}\sum_{j\in \dY}|\bP_\dS(Y_v = j| Y_u = i, v\in \dN_u) - \bP_\dT(Y_v = j| Y_u = i, v\in \dN_u)|
\end{align*}
Then, we take a weighted average of the TV distance for each class based on the label distribution of end nodes conditioned on an edge $\bP_\dU(Y_u|e_{uv}\in \dE_\dU)$ since classes that appear more often as a center node in the neighborhood may affect more in the structure shift. The CSS-src in the table indicates the weighted average by $\bP_\dS(Y_u|e_{uv}\in \dE_\dS)$ and CSS-tgt in the table indicates the weighted average by $\bP_\dT(Y_u|e_{uv}\in \dE_\dT)$, and CSS-both is the average of CSS-src and CSS-tgt.

The metric of LS is calculated as the total variation distance between the source and target label distribution as:
\begin{align*}
    TV(\bP_\dS(Y), \bP_\dT(Y)) = \frac{1}{2}\sum_{i\in \dY}|\bP_\dS(Y = i) - \bP_\dT(Y = i)|
\end{align*}

The shift metrics for each dataset are shown in the following tables.

\begin{table}[h!]
\caption{MAG dataset shift metrics}
\vspace{2mm}
\begin{center}
\begin{adjustbox}{width = 1\textwidth}
\begin{sc}
\begin{tabular}{lcccccccccc}
\toprule
  &   $US \rightarrow CN$ &  $US\rightarrow DE$  &     $US\rightarrow JP$    & $US\rightarrow RU$     & $US\rightarrow FR$ & $CN \rightarrow US$ &  $CN\rightarrow DE$  &     $CN\rightarrow JP$  & $CN\rightarrow RU$     & $CN\rightarrow FR$\\
\midrule
CSS-src        & $0.1639$ & $0.2299$ & $0.1322$    &$0.3532$    & $0.2530$     &  $0.2062$ & $0.1775$ &$0.1487$ &$0.2120$ & $0.1540$    \\
CSS-tgt        & $0.2062$ & $0.2217$ & $0.1438$    &$0.2866$    & $0.2854$     &  $0.1639$ & $0.2311$ &$0.1323$ &$0.2027$ & $0.2661$    \\
CSS-both        & $0.1850$ & $0.2258$ & $0.1380$    &$0.3199$    & $0.2692$     &  $0.1850$ & $0.2043$ &$0.1405$ &$0.2073$ & $0.2100$    \\
LS        & $0.2734$ & $0.1498$ & $0.1699$    &$0.3856$    & $0.1706$     &  $0.2734$ & $0.2691$ &$0.1522$ &$0.2453$ & $0.2256$    \\

\bottomrule
\label{table:magstats}
\end{tabular}
\end{sc}
\end{adjustbox}
\end{center}
\vskip -0.7cm
\end{table}

\begin{table}[h!]
\caption{HEP pileup dataset shift metrics}
\vspace{2mm}
\begin{center}
\begin{adjustbox}{width = 0.95\textwidth}
\begin{sc}
\begin{tabular}{lcccccccc}
\toprule
& \multicolumn{6}{c}{Pileup Conditions} & \multicolumn{2}{c}{Physical Processes} \\
 Domains  & $\text{PU}10 \rightarrow 30$    & $\text{PU}30 \rightarrow 10$ & PU$10\rightarrow 50$ & PU$50\rightarrow 10$ & PU$30\rightarrow 140$ & PU$140\rightarrow 30$ &$gg \rightarrow qq$      &     $qq \rightarrow gg$      \\
\midrule
CSS-src        & $0.1941$ & $0.1567$ & $0.2910$    &$0.2111$    & $0.1871$     &  $0.1307$  & $0.0232$     &  $0.0222$   \\
CSS-tgt        & $0.1567$ & $0.1941$ & $0.2111$    &$0.2910$    & $0.1307$     &  $0.1871$  & $0.0222$     &  $0.0232$   \\
CSS-both        & $0.1754$ & $0.1754$ & $0.2510$    &$0.2510$    & $0.1589$     &  $0.1589$  & $0.0227$     &  $0.0227$   \\
LS        & $0.2258$ & $0.2258$ & $0.3175$    &$0.3175$    & $0.1590$     &  $0.1590$  & $0.0348$     &  $0.0348$  \\

\bottomrule
\label{table:hepstats}
\end{tabular}
\end{sc}
\end{adjustbox}
\end{center}
\vskip -0.7cm
\end{table}

\begin{table}[h!]
\caption{Real dataset shift metrics}
\vspace{2mm}
\begin{center}
\begin{adjustbox}{width = 0.95\textwidth}
\begin{sc}
\begin{tabular}{lcccccccc}
\toprule
 & \multicolumn{2}{c}{1950-2007} & \multicolumn{2}{c}{1950-2009} & \multicolumn{2}{c}{1950-2011} & \multicolumn{2}{c}{DBLP and ACM}\\
 Domains  & $2014-2016$    & $2016-2018$ & $2014-2016$ & $2016-2018$ & $2014-2016$ & $2016-2018$  & $A\rightarrow D$ & $D\rightarrow A$\\
\midrule
CSS-src         & $0.2070$ & $0.2651$ & $0.1531$    &$0.2010$    & $0.1023$     &  $0.1443$   & $0.1400$     &  $0.2241$   \\
CSS-tgt         & $0.2404$ & $0.3060$ & $0.2043$    &$0.2737$    & $0.1504$     &  $0.2301$   & $0.2241$     &  $0.1400$   \\
CSS-both         & $0.2237$ & $0.2844$ & $0.1787$    &$0.2374$    & $0.1263$     &  $0.1872$   & $0.1820$     &  $0.1820$   \\
LS        & $0.2938$ & $0.4396$ & $0.2990$    &$0.4552$    & $0.2853$     &  $0.4438$   & $0.3435$     &  $0.3435$ \\

\bottomrule
\label{table:realstats}
\end{tabular}
\end{sc}
\end{adjustbox}
\end{center}
\vskip -0.7cm
\end{table}

\begin{table}[h!]
\caption{Synthetic CSBM dataset shift metrics}
\vspace{2mm}
\begin{center}
\begin{adjustbox}{width = 0.95\textwidth}
\begin{sc}
\begin{tabular}{lcccccccc}
\toprule
 & \multicolumn{2}{c}{CSS (only class ratio shift)} & \multicolumn{2}{c}{CSS (only degree shift)}  & \multicolumn{2}{c}{CSS (shift in both)} & \multicolumn{2}{c}{CSS + LS}\\
\midrule
CSS-src       & $0.1655$ & $0.3322$ & $0.0042$    &$0.0053$    & $0.1673$     &  $0.3308$  & $0.1777$     &  $0.2939$ \\
CSS-tgt       & $0.1655$ & $0.3322$ & $0.0042$    &$0.0053$    & $0.1673$     &  $0.3308$  & $0.1215$     &  $0.1840$ \\
CSS-both       & $0.1655$ & $0.3322$ & $0.0042$    &$0.0053$    & $0.1673$     &  $0.3308$  & $0.1496$     &  $0.2389$ \\
LS        & $0$ & $0$ & $0$    &$0$    & $0$     &  $0$  & $0.1650$     &  $0.2667$  \\

\bottomrule
\label{table:csbmstats}
\end{tabular}
\end{sc}
\end{adjustbox}
\end{center}
\vskip -0.7cm
\end{table}

\subsection{More results analysis}

\label{app:moreanalysis}
In this section, we will discuss more regarding our experimental results and provide some explanations of our \proj performance and comparison over the baselines.

\textbf{Synthetic Data}
As discussed in the main text, our major conclusion is that our \proj is practical for handling alignment by focusing only on the conditional neighborhood node label distribution to address class ratio shifts. Although \proj's performance is not the best among the baselines when there is a shift in node degree, we argue that in practice, ERM training alone is adequate under node degree shifts, especially when the graph size is large. Here, the graph size is only 6000—a small size in practical terms—and the ERM performance with a node degree shift ratio of 2 already achieved 99$\%$ accuracy. It should be perfect when the graph size is larger. Also, in the second setting with a degree shift, the degree ratio shift of 4 is relatively large, but the accuracy remains at 96$\%$. We expect that the decay should be negligible when the graph size is larger, often at least 10 times larger than 6000.

Regarding performance gains in addressing structure shifts, we observe that \projew demonstrates significant improvements, particularly in the second case of each scenario with larger degree shifts. Among the baselines, StruRW consistently outperforms others in different CSS scenarios, except in node degree shifts. This is expected since StruRW is specifically designed to handle CSS. Plus, in the synthetic CSBM data used here, the instability commonly associated with using hard pseudo-labels does not significantly affect performance due to easy classification task. However, compared to our \proj methods, StruRW still shows limited performance even with only CSS shifts. When both CSS and LS shifts occur, IWDAN emerges as the best baseline, as its algorithm addresses both conditional shifts and LS in non-graph problems effectively. In synthetic datasets, shifts are less complex than in real-world graph-structured data, allowing IWDAN to lead to empirical improvements. Our \projb outperforms all in scenarios involving CSS and LS shifts. By comparing \projew and \projlw, we found that when both CSS and LS occur, the impact of CSS often dominates, making \projew more effective than \projlw. However, this observation is based on our source graph's balanced label distribution and does not hold in the HEP pileup dataset when moving from highly imbalanced data (high PU conditions) to more balanced data (low PU conditions), which we will discuss later in relation to the Pileup dataset.

Another advantage of using synthetic dataset results is that they help us understand the experimental results on real datasets better. For example, by combining the shift statistics from Table~\ref{table:csbmstats} with the experimental results, we see that a CSS metric value around 0.16 does not significantly impact the performance, thus not clearly demonstrating the effectiveness of \proj. However, \proj methods show substantial benefits under larger shifts, with metric values around 0.3.

\textbf{MAG}
Overall, our \proj methods demonstrated significant advantages over the majority of baseline approaches, including the top-performing baseline, StruRW. When considering the relative improvement to ERM performance (as well as the performance of other baselines, except StruRW), there is an average relative benefit of over 45$\%$ when training on the US graph and nearly 100$\%$ when training on the CN graph. This substantial improvement corroborates our discussion regarding the existing gap, where current methods fall short in effectively addressing structure shifts. As detailed in the main text, our \projew methods not only surpass StruRW in performance but also yield additional benefits from handling LS, as the LS degree indicated in Table~\ref{table:magstats}. We believe the primary advantages stem from our principled approach to addressing CSS with $\mgamma$, which remains unbiased by LS, and the enhanced robustness afforded by using soft label predictions and regularized least square estimations. This also elucidates the shortcomings of IWDAN, a non-graph method for addressing conditional shift and LS, which underperforms under the MAG dataset conditions as discussed in the main text. 

We next explore the relationship between performance improvements and the degree of structure shift. The experimental results align closely with the CSS measurements shown in Table~\ref{table:magstats}. For example, the transitions from US to JP and CN to JP involve a smaller degree of CSS compared to other scenarios, resulting in relatively modest improvements. Similarly, generalizations between the US and CN also show fewer benefits. Conversely, the impact of LS is less evident in the outcomes associated with \projlw, as this approach alone yields only marginal improvements. However, when we evaluate the additional gains from LS mitigation provided by \projb in comparison to \projew, scenarios with larger LS (such as $US\rightarrow CN$, $CN\rightarrow US$, $US\rightarrow RU$, and $CN\rightarrow DE$ ) demonstrate more substantial benefits.

\textbf{Pileup Mitigation}
The most crucial discussions concerning HEP pileup datasets are detailed in the main text, particularly focusing on the distinct impacts of CSS and LS in transitions from high PU conditions to low PU conditions, and vice versa. This underscores that while the two directions have identical measures of LS, the direction of generalization is crucial. From a training perspective, it is clear that a model trained on a highly imbalanced dataset may neglect nodes in minor classes, leading to worse performance on more balanced datasets. To improve generalization, it is essential to adjust the classification loss to increase the model's focus on these minor nodes during training. This explains why \projew alone does not yield benefits in scenarios transitioning from high to low PU, and why \projlw becomes necessary. Conversely, when transitioning from low to high PU, \projew suffices to address CSS, as LS has a minimal effect on performance in this direction.

We then review baseline performance under the shift in pileup (PU) conditions. As noted in the main text, methods primarily addressing feature shifts, such as DANN, UDAGCN, and IWDAN, underperform, underscoring that PU conditions predominantly affect graph structure rather than node features. This observation aligns with the physical interpretation of PU shifts described in the dataset details in~\ref{app:morehepdata}. PU shift correlates with changes in the number of other collisions (OC) during collision events, directly influencing the OC
ratio and the pattern of node connections, as illustrated in Fig~\ref{fig:shift}. Given that node features are derived from particle labels (either OC or LC), the feature distribution remains largely unchanged despite variations in the OC to LC ratio. Consequently, feature shifts are minimal under PU conditions.

Consequently, the baselines like StruRW and SpecReg show some benefits over others in
regularizing and adjusting graph structure to handle structure shift. Specifically, SpecReg shows enhanced benefits during the transition from low PU to high PU, possibly due to its regularization of spectral smoothness, which mitigates edge perturbations beneficially under CSS conditions. Despite these improvements in the pileup dataset, SpecReg does not perform as well in other datasets characterized by CSS, such as MAG. This may be attributed to the fact that spectral regularization is more effective in scenarios with a limited variety of node connections, akin to the binary cases in the pileup dataset. However, it appears less capable of managing more complex shifts in neighborhood distribution involving multiple classes, as seen in datasets like MAG or Arxiv.

Conversely, StruRW achieves comparable performances to \projb in scenarios transitioning from high PU to low PU, predominantly influenced by LS. This effectiveness is likely due to the fact that their edge weights incorporate $\mw$, which includes $\malpha$ that implicitly contains the node label ratio. While our analysis suggests that using $\mw$ directly is not a principled approach for addressing CSS and LS, it proves beneficial in scenarios where LS significantly affects outcomes, providing a better calibration compared to approaches that do not address LS, like \projew. However, while StruRW holds an advantage over \projew, its performance still lags behind \projb, which offers a more systematic solution for both CSS and LS.

\textbf{Arxiv}
Results from the Arxiv datasets align well with expectations and the shift measures detailed in Table~\ref{table:realstats}. Notably, CSS is most pronounced when the source graph includes papers published before 2007, with experimental results showing the most substantial improvements under these conditions. In the scenario where papers from 2016-2018 are used for testing, both \projew and \projb outperform the baselines significantly, yet \projlw emerges as the superior variant. This aligns with the LS metrics reported, which indicate a significant LS in this context. A similar pattern is observed when training on papers from before 2011 and testing on those from 2016-2018, with \projlw achieving the best results.

For the target comprising papers from 2014-2016, our model continues to outperform baselines, albeit with a narrower margin compared to other datasets. In this case, not only does our method perform comparably, but all baselines also show similar performance levels, suggesting limited potential for improvements in this dataset. Furthermore, insights from synthetic experiments reveal that a CSS metric value around 0.16 does not lead to substantial performance degradation, which accounts for the moderate improvements over baselines in scenarios other than those using the source graph with pre-2007 papers.

In our evaluation of baseline performances, we note that StruRW emerges as the superior baseline, effectively handling CSS. In contrast, IWDAN tends to underperform relative to other baselines, which we attribute primarily to inaccuracies and instability in its label weight estimation. Designed for computer vision tasks where accuracy is typically high, IWDAN lacks mechanisms for regularization and robustness in its estimation processes, leading to its underperformance in our experiments involving tasks with a total of 40 classes. Meanwhile, the performance of other baselines is comparable to the ERM training.

\textbf{DBLP/ACM}
The generalization results between the DBLP and ACM datasets offer insights into the comparative effects of feature shift versus structure shift. As discussed in the main text, baselines focused on feature alignment tend to perform well in this dataset, suggesting that this dataset is predominantly influenced by feature shifts rather than structural shifts and that feature alignment can address the shift effectively. This trend also leads to non-graph methods performing comparably to, or even better than, graph-based methods due to the dominance of feature shifts.

In response to these observations, we integrated adversarial training into our method to align feature shifts and investigated whether additional benefits could be derived from mitigating structure shifts. Our analysis of the experimental results, in conjunction with the shift measures detailed in Table~\ref{table:realstats}, reveals a significant LS between these two datasets. Specifically, we note that the ACM graph exhibits a more imbalanced label distribution compared to the DBLP graph. This finding aligns with the experimental outcomes, where \projlw emerges as the most effective model and IWDAN as the best baseline when training on ACM and testing on DBLP. Both methods are adept at handling LS, supporting our earlier assertion that LS plays a crucial role when transitioning from an imbalanced dataset to a more balanced one. Conversely, in the transition from DBLP to ACM, where LS has a lesser impact, \projb proves to be the most effective.

\end{document}